%% file: main_emnlp.tex
\documentclass[11pt]{article}

\usepackage[final]{acl}

\usepackage{times}
\usepackage{latexsym}

\usepackage[T1]{fontenc}

\usepackage[utf8]{inputenc}

\usepackage{microtype}

\usepackage{inconsolata}

\usepackage{graphicx}

\usepackage{makecell}
\usepackage{amsmath}
\usepackage{amssymb}
\usepackage{mathtools}
\usepackage{amsthm}
\usepackage{algorithm}
\usepackage{algorithmic}
\usepackage{wrapfig} 
\usepackage[most]{tcolorbox}
\usepackage{xcolor}
\usepackage{threeparttable}
\usepackage{multirow} 
\usepackage{booktabs}
\usepackage{subcaption}
\usepackage{hyperref}
\usepackage{pifont}
\usepackage{svg}
\newcommand{\cmark}{\ding{51}}%
\newcommand{\xmark}{\ding{55}}%
\usepackage[table]{xcolor}
\usepackage{svg}
\usepackage{float}
\usepackage{placeins}
\newcommand{\MEE}{\mathrm{MEE}}
\newcommand{\MES}{\mathrm{MES}}
\theoremstyle{plain}
\newtheorem{theorem}{Theorem}[section]

\newtheorem{lemma}[theorem]{Lemma}

\theoremstyle{definition}

\theoremstyle{remark}

\title{Beyond Truncation: Rethinking LLM Decoding as Ensemble Pruning}

\author{
  \textbf{Dunyao Xue\textsuperscript{1}},
  \textbf{Chengshuo Du\textsuperscript{1}},
  \textbf{Zhengbo Wang\textsuperscript{1}},
  \textbf{Wenlin Dai\textsuperscript{1,2,\textdagger}},
  \textbf{Cheng Meng\textsuperscript{1,3,\textdagger}}
\\
\\
  \textsuperscript{1}Institute of Statistics and Big Data, Renmin University of China, Beijing, China
\\
  \textsuperscript{2}Big Data and Responsible Artificial Intelligence for National Governance, Renmin University of China, Beijing, China
\\
  \textsuperscript{3}Center for Applied Statistics, Institute of Statistics and Big Data, Renmin University of China, Beijing, China
\\
  \texttt{\{xuedunyao1202,duchengshuo,zhengbowang,wenlin.dai,chengmeng\}@ruc.edu.cn}
}

\begin{document}
\maketitle
\begingroup
\renewcommand{\thefootnote}{\textdagger}
\begin{NoHyper}
\footnotetext{Corresponding authors.}
\end{NoHyper}
\endgroup
\begin{abstract}
We introduce Mahalanobis-Ensemble Decoding (ME-Decoding), a novel Large Language Model (LLM) decoding framework that frames candidate token selection as ensemble pruning. Existing selection strategies rely predominantly on scalar probabilities, ignoring geometric semantic relationships and causing candidate redundancy. Meanwhile, current geometry-aware methods often require complex optimization or directly reweighting the original token probabilities, leading to significant computational overhead or inference instability. To address this, we formulate decoding as a subset optimization problem using a Mahalanobis distance-driven objective to enhance semantic diversity while preserving high probabilities. Specifically, we dynamically discount redundant generation paths using a token similarity matrix, constructed via an adaptive-bandwidth kernel over token embeddings. We further devise an efficient greedy selection algorithm with near-linear complexity in the candidate size under early stopping, while establishing its theoretical approximation guarantees. This renders ME-Decoding a robust, plug-and-play module with negligible inference overhead. Extensive experiments across diverse reasoning and generation tasks demonstrate that our method consistently achieves strong performance.
\end{abstract}

\section{Introduction}

Large Language Models (LLMs) have demonstrated strong capabilities in mathematical reasoning, instruction following, and open-ended generation \citep{brown2020language, touvron2023llama, chowdhery2023palm, achiam2023gpt}. 
Besides model scaling and post-training alignment, the inference-time decoding strategy also plays a crucial role in determining generation quality. 
At each step, an autoregressive LLM produces a next-token distribution, from which the decoder selects or samples the next token. 
Deterministic methods such as greedy decoding and beam search \citep{holtzman2019curious} favor high-probability continuations but often produce repetitive or overly conservative outputs. 
Sampling-based methods introduce stochasticity and improve generation diversity, but they may also assign non-negligible probability to low-quality tokens, increasing inference uncertainty and potentially leading to illogical continuations or hallucinations.

To control this quality--diversity trade-off, many decoding methods perform probability-based truncation and reshaping of the next-token distribution. 
Classical approaches such as Top-$k$ and nucleus sampling restrict sampling to high-probability tokens \citep{fan2018hierarchical, holtzman2019curious}, while later methods such as Min-$p$, entropy-aware sampling, and $p$-less adapt the truncation rule according to model confidence or distributional statistics \citep{hewitt2022truncation, nguyen2025turning, tan2025p}. 
Despite their effectiveness, these methods mainly operate on token probabilities and treat candidates as independent categorical outcomes. 
As a result, they overlook the semantic geometry among tokens, which may retain redundant candidates and limit the effectiveness of the selected sampling support \citep{yang2026decoding}.

\begin{table}[htbp]
\centering
\footnotesize
\setlength{\tabcolsep}{2.pt}
\renewcommand{\arraystretch}{0.95}
\caption{Comparison of different decoding methods.}
\label{tab:decoding_comparison}
\begin{threeparttable}
\begin{tabular}{@{}lcccc@{}}
\toprule
\multirow{2}{*}{\textbf{Method}} 
& \textbf{Decoding} 
& \textbf{Geometry} 
& \textbf{Adaptive} 
& \textbf{Redundancy} \\
& \textbf{Strategy} 
& \textbf{Aware\tnote{1}} 
& \textbf{Selection} 
& \textbf{Control\tnote{2}} \\
\midrule
Greedy  & Determin.    & \xmark & \xmark & \xmark \\
Top-$k$   & Prob. Trunc. & \xmark & \xmark & \xmark \\
Min-$p$   & Prob. Trunc. & \xmark & \xmark & \xmark \\
$p$-less  & Prob. Trunc. & \xmark & \cmark & \xmark \\
Top-$W$   & Dist. Match. & \cmark & \cmark & \xmark \\
CraEG   & Reweighting  & \cmark & \xmark & \cmark \\
\midrule
\textbf{Ours} 
& \textbf{Ensemble}
& \cmark
& \cmark
& \cmark \\
\bottomrule
\end{tabular}
\begin{tablenotes}
\footnotesize
\item[1] \textbf{Embedding Geometry}: Uses token embedding geometry in the decoding process.
\item[2] \textbf{Redundancy Control}: Explicitly controls redundancy among candidate tokens.
\end{tablenotes}
\end{threeparttable}
\end{table}
Although recent geometry-aware methods incorporate token-space structure, they continue to face significant challenges. For example, Top-$W$ formulates a Wasserstein-regularized distribution-matching problem \citep{davoodi2026geometry}, which requires approximating the Wasserstein objective and careful hyperparameter tuning, while not explicitly optimizing redundancy within the selected token support. Alternatively, CraEG penalizes tokens in crowded regions via a lightweight reweighting mechanism \citep{yang2026decoding}. However, the approach still relies on an auxiliary decoding method for post-processing, making its behavior dependent on the properties of the downstream method. These limitations motivate our key question:

\begin{center}
\begin{tcolorbox}[colframe=black!50!white, 
    colback=gray!5!white, 
    boxrule=0.5mm, 
    arc=5pt,width=7.6cm] 
\centering
\textit{How can we design  a simple and efficient framework that selects a compact yet informative token set while preserving high-confidence candidates?}
\end{tcolorbox}
\end{center}

\begin{figure*}[t]
    \centering
    \includegraphics[width=\textwidth]{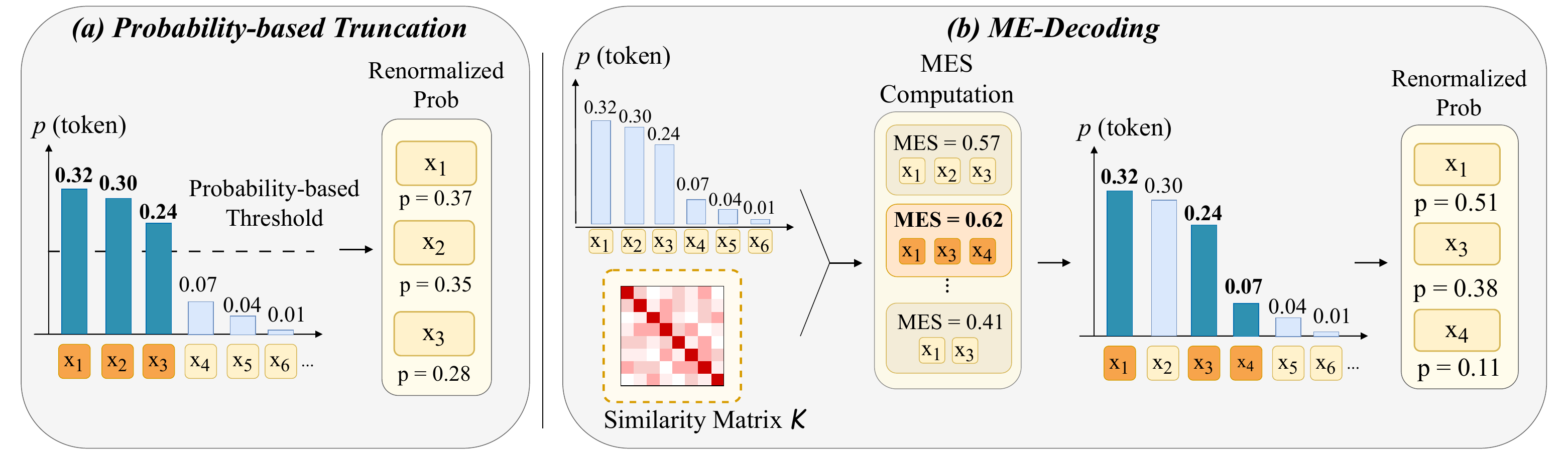}
    \caption{
Comparison between probability-based truncation methods (left) and our ME-Decoding framework (right). 
ME-Decoding extends standard probability-only token selection by incorporating embedding-based similarity into the MES score and selecting a compact geometry-aware subset for sampling.}
\label{fig:flowchart}
\end{figure*}

In this work, we address this question from the perspective of ensemble pruning, a paradigm that entails selecting a compact subset of learners from a larger pool to optimize predictive outcomes. This paradigm closely mirrors the token selection problem in LLM decoding. Extensive research in ensemble learning shows that effective ensembles should balance individual accuracy with collective complementarity \citep{krogh1994neural, opitz1999popular}.
Consequently, simply selecting the highest-scoring learners may introduce redundancy and hurt overall performance \citep{zhou2002ensembling, li2012diversity}. Similarly, LLM decoding encounters an analogous challenge of token redundancy. Building on this insight, we introduce ME-Decoding, a novel framework that incorporates token geometry to reformulate decoding as a subset maximization problem driven by the Mahalanobis distance. By approximately solving this objective with an efficient greedy procedure, ME-Decoding mitigates redundancy within the selected token set during generation with negligible inference overhead, thereby improving token selection quality while maintaining reasoning performance.


Our contributions are summarized as follows:
\begin{itemize}
    \item We reformulate LLM decoding through the novel lens of ensemble pruning.
    
    \item We define a Mahalanobis distance-driven objective to adaptively select a compact, informative candidate set while preserving token probabilities.

    \item We devise an efficient greedy selection algorithm with candidate-linear complexity under early stopping, while establishing its trajectory unimodality and theoretical approximation guarantees.
    
    \item Extensive experiments across diverse reasoning and generation tasks show that our method consistently outperforms strong existing baselines.
\end{itemize}

\section{Motivation}

\subsection{Background of LLM Decoding}
Large Language Models (LLMs) generate text autoregressively by predicting the next token from the preceding context. At each step, an LLM outputs a probability distribution $\mathcal{P}\in\Delta^{|\Omega|-1}$ over the full vocabulary $\Omega$. To avoid sampling from the unreliable long tail, modern decoding methods usually first form a candidate pool $\mathcal V_t\subseteq\Omega$, then select a final subset $S\subseteq\mathcal V_t$ and sample from the renormalized distribution:
\begin{equation*}
    \boldsymbol{p}_S(i) =
    \frac{p_{i}\mathbf{1}\{i\in S\}}
    {\sum_{j\in S}p_{j}}.
\end{equation*}

This process can be viewed as a subset selection problem aiming to preserve the original distribution under a probability-based criterion:
\begin{equation*}
\begin{split}
    S^\star &= \arg\min_{S\subseteq\mathcal V_t} f(\boldsymbol{p}_S), \\
    \text{s.t.} &\quad S=\{i:p_{i}\geq \tau(\mathcal{P})\},
\end{split}
\end{equation*}
where $f(\cdot)$ denotes a generalized decoding objective and $\tau(\mathcal{P})$ is a probability-based truncation threshold.

Most prevailing decoding methods follow this probability-driven paradigm. Top-$k$ and nucleus sampling truncate the distribution using fixed probability or cumulative-mass thresholds, while adaptive methods such as $p$-less and entropy-aware sampling adjust the threshold based on distributional statistics. 
Despite different truncation rules, these methods remain probability-centric and treat candidate tokens as independent outputs. 
They therefore overlook semantic dependencies among tokens, motivating us to incorporate inter-token semantics into the subset selection objective.

\subsection{Connection to Ensemble Pruning}

Ensemble pruning aims to select a compact sub-ensemble that preserves or improves the predictive performance of the full ensemble. Classical error analysis \citep{breiman2001random} characterizes ensemble performance as a trade-off between individual accuracy and inter-model diversity. Following this idea, \citet{zhang2006ensemble} formulate pruning as a quadratic subset-selection problem. Given a binary error indicator matrix $\boldsymbol{E}$, they construct the error co-occurrence matrix $\boldsymbol{C}=\boldsymbol{E}^{\top}\boldsymbol{E}$, where diagonal entries measure individual errors and off-diagonal entries measure shared failures. After normalization, the pruning objective can be written as
\begin{equation}
\min_{\boldsymbol{s}} \ \boldsymbol{s}^{\top}\widehat{\boldsymbol{C}}\boldsymbol{s}
\quad \text{s.t.}
\sum_i s_i = N,\ s_i\in\{0,1\}.
\label{eq:ensemble_pruning}
\end{equation}
This objective selects a size-$N$ sub-ensemble by jointly penalizing individual inaccuracy and pairwise redundancy.

This paradigm naturally parallels token selection in LLM decoding. Candidate tokens can be viewed as ensemble members: their probabilities measure individual confidence, while their embedding-space similarities measure semantic redundancy. Therefore, an ideal decoding subset should retain high-likelihood tokens while suppressing redundant candidates, mirroring the accuracy--diversity trade-off in ensemble pruning.

\subsection{Generalizing Ensemble Pruning to LLM Decoding via
Mahalanobis Distance
}
The quadratic objective in Eq.~\ref{eq:ensemble_pruning} leverages second-order statistics to penalize redundant candidates. Motivated by this, we aim to extend this diversity-aware framework to LLM decoding. To achieve this, we observe that for any representation vector $\boldsymbol{p}$ and a symmetric positive definite matrix, this quadratic structure naturally corresponds to the Mahalanobis distance.

In machine learning, the Mahalanobis distance is widely used to construct discriminative representations, such as in metric learning, out-of-distribution detection, and representation regularization \citep{weinberger2009distance,lee2018simple,wan2018rethinking,2026breaking}.
Formally, given two vectors $\boldsymbol{x},\boldsymbol{y}\in\mathbb{R}^d$ and a symmetric positive definite matrix $\boldsymbol{M}\succ0$, it is defined as
\begin{equation*}
    \mathcal{D}_M(\boldsymbol{x},\boldsymbol{y};\boldsymbol{M})
    :=
    \sqrt{
    (\boldsymbol{x}-\boldsymbol{y})^\top
    \boldsymbol{M}^{-1}
    (\boldsymbol{x}-\boldsymbol{y})
    }.
    \label{eq:mahalanobis_general}
\end{equation*}
In particular, setting $\boldsymbol{y}=\mathbf{0}$ and letting $\boldsymbol{M}$ be the token similarity matrix $\boldsymbol{K}$ gives the Mahalanobis norm
\begin{equation*}
    \|\boldsymbol{x}\|_{\boldsymbol{K}}
    :=
    \sqrt{\boldsymbol{x}^{\top}\boldsymbol{K}^{-1}\boldsymbol{x}}.
    \label{eq:mahalanobis_norm}
\end{equation*}

Mathematically, by explicitly incorporating the inverse matrix $\boldsymbol{K}^{-1}$, the Mahalanobis norm inherently downweights correlated directions and separates information along non-redundant axes. This geometric property aligns perfectly with our overarching goal: to suppress repetitive generation paths and enhance token diversity during decoding.

\section{Method: ME-Decoding}
In this section, we propose Mahalanobis-Ensemble Decoding (ME-Decoding), a geometry-aware decoding framework inspired by ensemble pruning. 
ME-Decoding selects a compact token subset by optimizing a Mahalanobis-style objective that combines token probabilities with embedding-based similarities.
The overall structure of ME-Decoding is illustrated on the right side of Figure~\ref{fig:flowchart}.
\subsection{Mahalanobis-Ensemble Score.}
Inspired by ensemble pruning, we view the candidate tokens at each decoding step as a pool of weak semantic predictors. Instead of retaining a fixed number of high-probability tokens, our goal is to construct a compact token ensemble that preserves high-confidence candidates while reducing redundancy in the token embedding space. At each decoding step, ME-Decoding is applied to a probability-based candidate pool $\mathcal V_t\subseteq\Omega$; for simplicity, we omit $t$ and write $\mathcal V$ hereafter.

We first define the \emph{Mahalanobis-Ensemble Energy} ($\MEE$) for a selected token subset \(\mathcal{S}\subseteq\mathcal{V}\) as
\begin{equation*}
    \MEE(\mathcal{S})
    =
    \boldsymbol{p}_S^\top\boldsymbol{K}_\mathcal{S}^{-1}\boldsymbol{p}_\mathcal{S},
    \label{eq:mee}
\end{equation*}
where \(\boldsymbol{p}_\mathcal{S}\) denotes the vector of token probabilities indexed by \(\mathcal{S}\), and \(\boldsymbol{K}_\mathcal{S}\) is the token similarity matrix restricted to \(\mathcal{S}\). The inverse matrix \(\boldsymbol{K}_\mathcal{S}^{-1}\) discounts redundant tokens and rewards subsets with high collective quality.

However, unlike standard ensemble pruning, the optimal subset size in decoding is unknown a priori. This requires us to maximize the $\MEE$ of the selected tokens while avoiding unnecessarily large subsets, thereby filtering out redundant tokens that contribute only marginally to the overall $\MEE$. To this end, we define the \emph{Mahalanobis-Ensemble Score} ($\MES$) as follows:
\begin{equation*}
    \MES_{\lambda}(\mathcal{S})
    =
    \frac{\MEE(\mathcal{S})}{\left[1+\lambda H_2^{T}( \boldsymbol{p})\right]^{|\mathcal{S}|}},
    \label{eq:mes_objective}
\end{equation*}
where $H_2^T(p)=1-\sum_{i\in\mathcal V}p_i^2$ is the order-2 Tsallis entropy measuring the uncertainty of the next-token distribution, and $\lambda>0$ controls the entropy-dependent size penalty.

Intuitively, the numerator rewards tokens with large non-redundant contributions, while the denominator discourages excessive subset expansion. When the distribution is flat, the unregularized $\MEE$ may keep increasing by accumulating weak marginal gains from many uncertain candidates. The entropy-dependent penalty raises the inclusion threshold, retaining only tokens with sufficiently large non-redundant contributions. 
Consequently, the decoding objective is given by:
\begin{equation}
    \mathcal{S}^\star
    =
    \arg\max_{\mathcal{S}\subseteq\mathcal{V}}
    \MES_{\lambda}(\mathcal{S}).
    \label{eq:mes_maximization}
\end{equation}
By optimizing the $\MES$, we can effectively extract a highly representative token subset from the candidate distribution.

\subsection{Construction of Similarity  Matrix}
To measure token similarity in the embedding space while maintaining numerical stability, we construct $\boldsymbol{K}$ using a Gaussian kernel on normalized token embeddings. Let \(\boldsymbol{e}_{i}\in\mathbb R^d\) denote the normalized embedding of token \(i\). We define
\begin{equation}
    C_{ij}=1-\boldsymbol{e}_{i}^\top \boldsymbol{e}_{j},
\qquad
K_{ij}=\exp\{-C_{ij}/\epsilon\}.
\label{eq:sim kernel}
\end{equation}

Here, \(C_{ij}\) measures the semantic distance between tokens, and \(\epsilon>0\) is a bandwidth parameter controlling the smoothness of the kernel.

In our implementation, the bandwidth $\epsilon$ is chosen adaptively according to the probability-weighted semantic dispersion of the candidate tokens:
$$
\epsilon
=
\frac{1}{2}\sum_{i,j\in\mathcal{V}}
p_i p_j C_{ij},
$$
where $\sum_{i,j\in\mathcal{V}}
p_i p_j C_{ij}$ is the expected pairwise semantic distance between tokens sampled from the candidate distribution. When the candidate tokens are semantically concentrated, the adaptive bandwidth becomes smaller, yielding a localized kernel and making the selection more probability-driven. Conversely, when the candidates are semantically dispersed, the bandwidth becomes larger, preserving semantic relations over a broader range and making the selection more geometry-aware. Thus, the adaptive bandwidth balances probability and semantic structure according to the local geometry of the decoding distribution.

\subsection{Optimization Strategy}
\begin{algorithm}[t!]
\caption{Greedy Algorithm for ME-Decoding}
\label{alg:greedy_mes_decoding}
\begin{algorithmic}[1]
\STATE {\bfseries Input:} Candidate token pool $\mathcal V$ with $N=|\mathcal V|$; candidate token scores $\boldsymbol{p}\in\mathbb{R}^{N}$; normalized candidate token embeddings $\boldsymbol{E}=[\boldsymbol{e}_1,\ldots,\boldsymbol{e}_N]^\top\in\mathbb{R}^{N\times d}$; $\MES$ penalty $\lambda$.
\STATE {\bfseries Initialize} $\mathcal{S}=\varnothing$.

\STATE Set $c_\lambda\gets 1+\lambda H_2^T(\boldsymbol{p})$.

\STATE Select the first token $j_1=\arg\max_i p_{i}$ and update $\mathcal{S}\gets\{j_1\}$.

\STATE Compute $\boldsymbol{K}_{\mathcal{S},\mathcal{S}}$ using Eq.~\ref{eq:sim kernel}.
\STATE $\boldsymbol{R}\gets (\boldsymbol{K}_{\mathcal{S},\mathcal{S}}^{1/2})^{-1}$, \quad $\boldsymbol{z}\gets \boldsymbol{R}\boldsymbol{p}_{\mathcal{S}}$.
\STATE $\MES^{g}\gets \|\boldsymbol{z}\|_2^2/c_\lambda^{|\mathcal{S}|}$.

\FOR{$t=1$ {\bfseries to} $N-1$}
\FOR{$j \in \{1,\ldots,N\}\setminus\mathcal{S}$}
\STATE Compute $\boldsymbol{\beta}_j\gets [K_{ij}]_{i\in\mathcal{S}}$ using Eq.~\ref{eq:sim kernel}, and set $b_j\gets \boldsymbol{K}_{jj}$.
\STATE $\boldsymbol{\alpha}_j\gets \boldsymbol{R}\boldsymbol{\beta}_j$, \quad $r_j\gets (b_j-\|\boldsymbol{\alpha}_j\|_2^2)^{-1/2}$.
\STATE Compute the $\MEE$ after adding token $j$:
\STATE $c_j\gets \|\boldsymbol{z}\|_2^2+\left[r_j(p_{j}-\boldsymbol{\alpha}_j^\top \boldsymbol{z})\right]^2$.
\ENDFOR

\STATE Select $j_\star=\arg\max_{j\notin\mathcal{S}} c_j$, let $c^\star=c_{j_\star}$.

\STATE Compute the candidate $\MES$ score:
\STATE $\MES^{g}_{\mathrm{cand}}\gets c^\star/c_\lambda^{|\mathcal{S}|+1}$.

\IF{$\MES^{g}_{\mathrm{cand}}\leq \MES^{g}$}
    \STATE {\bfseries break}
\ENDIF

\STATE $\boldsymbol{\gamma}\gets -r_{j_\star}\boldsymbol{R}^{\top}\boldsymbol{\alpha}_{j_\star}$.
\STATE $v_{j_\star}\gets r_{j_\star}\left(p_{j_\star}-\boldsymbol{\alpha}_{j_\star}^{\top}\boldsymbol{z}\right)$.

\STATE Update $\boldsymbol{R}$ and $\boldsymbol{z}$:
\[
    \boldsymbol{R}\gets
    \begin{bmatrix}
        \boldsymbol{R} & \boldsymbol{0}\\
        \boldsymbol{\gamma}^{\top} & r_{j_\star}
    \end{bmatrix},
    \quad
    \boldsymbol{z}\gets
    \begin{bmatrix}
        \boldsymbol{z}\\
        v_{j_\star}
    \end{bmatrix}.
\]

\STATE $\mathcal{S}\gets \mathcal{S}\cup\{j_\star\}$, \quad
$\MES^{g}\gets \MES^{g}_{\mathrm{cand}}$.

\ENDFOR

\STATE {\bfseries Output:} Selected token set $\mathcal{S}$.
\end{algorithmic}
\end{algorithm}
Directly solving the maximization problem formulated in Eq.~\ref{eq:mes_maximization} is an NP-hard combinatorial task. 
To maintain computational feasibility, we use a greedy forward-selection procedure and maintain the inverse Cholesky factor of \(\boldsymbol{K}_{\mathcal{S}}^{-1}\) incrementally \citep{golub2013matrix}, which avoids repeated matrix inversions during subset construction.

The greedy selection in Algorithm~\ref{alg:greedy_mes_decoding} is conceptually related to Orthogonal Matching Pursuit (OMP) \citep{pati1993orthogonal}. At each step, $(p_j-\boldsymbol{\alpha}_j^\top z)$ represents the conditional residual score of token $j$ after accounting for its correlation with the selected tokens, while $r_j$ normalizes it by the corresponding conditional variance. The marginal contribution is therefore measured by $\left[r_j(p_j-\boldsymbol{\alpha}_j^\top z)\right]^2$. The algorithm selects the token with the largest contribution and accepts it only when the penalized $\MES$ objective improves, favoring high-probability tokens while discounting redundant candidates.

\subsection{Theoretical Properties}
Furthermore, to rigorously validate the effectiveness of our proposed strategy, we establish theoretical properties of the greedy algorithm, showing its unimodality along the search path and deriving an approximation bound to the global optimum.

The following theorem states that, under a conditioning assumption on the token-similarity matrix, the greedy $\MES$ sequence cannot increase again once it stops increasing.

\begin{theorem}[Greedy Unimodality of Mahalanobis-Ensemble Decoding]
\label{thm:mes_saturation}
Let $|\mathcal{V}|=N$ and $\MES_t^g=\MES_{\lambda}(\mathcal{S}_t)$, where $\{\mathcal{S}_t\}_{t=0}^{N}$ is the greedy path generated by Algorithm~\ref{alg:greedy_mes_decoding}. For any $T\subseteq\mathcal{V}$, let $\boldsymbol{K}_{T}=(K_{ij})_{i,j\in T}$ and define
$
\kappa_N(\boldsymbol{K})
=
\max_{T\subseteq\mathcal{V},\ |T|\leq N}
\frac{\lambda_{\max}(\boldsymbol{K}_{T})}{\lambda_{\min}(\boldsymbol{K}_{T})}.
$
If $\kappa_N(\boldsymbol{K})\le 1+\lambda H_2^T(p)$,
once for some $t<N$,
\[
\MES_{t+1}^g\leq \MES_t^g,
\]
then for all $s\geq t$,
$
\MES_{s+1}^g\leq \MES_s^g.
$
Consequently, with 
$
\tau=\min\{t:\MES_{t+1}^g\leq \MES_t^g\},
$
we have
\[
\MES_{\tau}^g
=
\max_{0\leq r\leq N}\MES_r^g.
\]
\end{theorem}

Theorem~\ref{thm:mes_saturation} justifies the early stopping rule in Algorithm~\ref{alg:greedy_mes_decoding}: the greedy search can terminate once the $\MES$ score drops, without traversing all candidates. Appendix~\ref{app:unimodality_probe} complements this sufficient-condition result by disabling early stopping and recording complete greedy trajectories on Qwen2.5-1.5B. All 512 trajectories on each of GSM8K and GPQA are unimodal, providing empirical support for the stopping behavior characterized by Theorem~\ref{thm:mes_saturation}. The next theorem compares the stopped greedy value with the global optimum, showing that it provides an effective approximation to the optimal subset.

\begin{theorem}[Approximation Guarantee for Mahalanobis-Ensemble Decoding]
\label{thm:mes_stopping_approx}
Let
\[
\MES_{\lambda}^{\star}
=
\max_{\mathcal{S}\subseteq\mathcal{V}}
\MES_{\lambda}(\mathcal{S}),
\qquad |\mathcal{V}|=N.
\]
Let $\{\mathcal{S}_t\}_{t=0}^{N}$ be the greedy path generated by Algorithm~\ref{alg:greedy_mes_decoding}, and let $\tau$ be the stopping time defined in Theorem~\ref{thm:mes_saturation}. Under the assumptions of Theorem~\ref{thm:mes_saturation}, we have
\[
\MES_{\tau}^{g}
\geq
\left(1-\exp\{-\lambda_{\min}(\boldsymbol{K},N)\}\right)
\MES_{\lambda}^{\star},
\]
where
$\lambda_{\min}(\boldsymbol{K},r)
\triangleq
\min_{T\subseteq\mathcal{V},\ |T|=r}
\lambda_{\min}(\boldsymbol{K}_{T})$
denotes the smallest eigenvalue among all $r\times r$ principal submatrices of $\boldsymbol{K}$.
\end{theorem}

Theorem~\ref{thm:mes_stopping_approx} establishes the relation between the optimal greedy value and the globally optimal $\MES$ value. The approximation factor depends on the restricted minimum eigenvalue of $\boldsymbol{K}$, which reflects the non-redundancy and conditioning of the selected token representations.

\section{Experiments}
\begin{table*}[t!]
\centering
\footnotesize
\setlength{\tabcolsep}{5pt}
\begin{tabular}{lccc|ccc|ccc|c}
\toprule
& \multicolumn{3}{c|}{Qwen3-4B-Inst.}
& \multicolumn{3}{c|}{Phi-4-mini-Inst.}
& \multicolumn{3}{c|}{Mistral-7B-Inst.}
&  \\
\cmidrule(lr){2-4} \cmidrule(lr){5-7} \cmidrule(lr){8-10}
Method
& $T=1.0$ & $T=1.5$ & $T=2.0$
& $T=1.0$ & $T=1.5$ & $T=2.0$
& $T=1.0$ & $T=1.5$ & $T=2.0$
& Avg. \\
\midrule
Min-$p$
& 70.74 & 62.02 & 54.06
& 79.68 & 70.81 & 30.86
& 48.67 & 36.32 & 17.97
& 52.35 \\

Top-$p$
& 68.01 & 55.19 & 21.00
& 80.06 & 11.75 & 0.30
& 48.90 & 23.05 & 0.45
& 34.30 \\

$p$-less
& \underline{78.92} & 72.71 & 63.15
& \underline{84.05} & 83.09 & 69.52
& \underline{54.06} & 50.80 & 46.47
& 66.97 \\

Top-$H$
& 76.80 & 67.17 & 58.53
& 81.65 & 73.77 & 34.65
& 52.01 & 46.78 & 22.52
& 57.10 \\

Top-$W$
& 77.71 & \underline{76.65} & \underline{74.98}
& 82.64 & \underline{83.24} & \underline{82.18}
& 53.98 & \underline{53.15} & \underline{51.02}
& \underline{70.62} \\
\rowcolor{gray!18} 
\textbf{Ours}
& \textbf{80.67} & \textbf{79.76} & \textbf{80.06}
& \textbf{84.08} & \textbf{84.23} & \textbf{82.41}
& \textbf{55.34} & \textbf{53.45} & \textbf{53.90}
& \textbf{72.66} \\
\bottomrule
\end{tabular}
\caption{GSM8K accuracy (\%) across different temperatures and decoding methods. The Avg. column reports the average accuracy over all models and temperatures. The best result is in \textbf{bold} and the second best is \underline{underlined}.}
\label{tab:gsm8k_results}
\end{table*}

\begin{table*}[t!]
\centering
\footnotesize
\setlength{\tabcolsep}{5pt}
\begin{tabular}{lccc|ccc|ccc|c}
\toprule
& \multicolumn{3}{c|}{Qwen3-4B-Inst.}
& \multicolumn{3}{c|}{Phi-4-mini-Inst.}
& \multicolumn{3}{c|}{Mistral-7B-Inst.}
&  \\
\cmidrule(lr){2-4} \cmidrule(lr){5-7} \cmidrule(lr){8-10}
Method
& $T=1.0$ & $T=1.5$ & $T=2.0$
& $T=1.0$ & $T=1.5$ & $T=2.0$
& $T=1.0$ & $T=1.5$ & $T=2.0$
& Avg. \\
\midrule
Min-$p$
& 34.38 & 32.59 & 35.04
& 27.01 & 29.02 & 26.56
& 26.12 & 28.12 & 25.22
& 29.34 \\

Top-$p$
& 34.60 & \underline{35.49} & 30.58
& 31.25 & 27.01 & 15.40
& 25.89 & 26.41 & 14.06
& 26.74 \\

$p$-less
& \underline{35.27} & 34.82 & \underline{35.49}
& \underline{33.26} & 33.26 & 27.90
& 27.68 & 27.01 & 24.78
& 31.08 \\

Top-$H$
& 34.38 & 33.93 & 34.15
& \textbf{34.38} & \underline{33.93} & 29.91
& \underline{28.12} & 27.23 & 23.88
& 31.10 \\

Top-$W$
& 34.82 & \textbf{36.83} & 35.27
& 31.70 & 29.46 & \underline{31.47}
& 27.23 & \textbf{29.02} & \underline{28.23}
& \underline{31.56} \\

\rowcolor{gray!18} 
\textbf{Ours}
& \textbf{35.49} & \underline{35.49} & \textbf{35.71}
& \textbf{34.38} & \textbf{36.16} & \textbf{33.93}
& \textbf{28.35} & \underline{28.79} & \textbf{28.35}
& \textbf{32.96} \\
\bottomrule
\end{tabular}
\caption{GPQA accuracy (\%) across different temperatures and decoding methods. The Avg. column reports the average accuracy over all models and temperatures. The best result is in \textbf{bold} and the second best is \underline{underlined}.}
\label{tab:gpqa_results}
\end{table*}
\subsection{Experimental Details}
\label{sec:experiment_details}

\paragraph{Models and Baselines.}
We evaluate ME-Decoding on three instruction-tuned language models: Qwen3-4B-Instruct \citep{yang2025qwen3}, Phi-4-mini-Instruct \citep{abouelenin2025phi}, and Mistral-7B-Instruct \citep{jiang2023mistral7b}. We compare our method with several representative decoding methods, including Min-$p$ \citep{nguyen2025turning}, Top-$p$ \citep{holtzman2019curious}, $p$-less \citep{tan2025p}, Top-$H$ \citep{baghaei2026top}, and Top-$W$  \citep{davoodi2026geometry}. For a fair comparison, all methods are evaluated under the same prompts, temperature settings, maximum generation length, and stopping criteria. Unless otherwise specified, we evaluate all methods at temperatures $T\in\{1.0,1.5,2.0\}$. All reported experiments use $N=512$.
We additionally implement CraEG \citep{yang2026decoding} and combine its post-softmax reweighting with $p$-less, following its plug-in formulation. Repeated results for selected methods are reported in Appendix~\ref{app:repeated_reasoning}.

\paragraph{Benchmarks and Metrics.}
We consider two types of tasks. First, for reasoning benchmarks, we evaluate on GSM8K \citep{cobbe2021training} and GPQA \citep{rein2023gpqa}, where performance is measured by answer accuracy after extracting the final answer. These benchmarks test whether a decoding method can maintain reliable reasoning performance under different sampling temperatures. Second, for instruction-following and chat-style generation, we evaluate on AlpacaEval \citep{alpaca_eval} and MT-Bench \citep{zheng2023judgingllmasajudgemtbenchchatbot}. For AlpacaEval, we report the candidate win-rate (\%), and for MT-Bench, we report the average judge score. Since these tasks involve open-ended generation, we use DeepSeek-V4-Pro \citep{deepseekai2026deepseekv4} as the judge to compare the generated responses under a fixed evaluation protocol.
For reproducibility, AlpacaEval compares each response with a fixed reference for the same prompt, whereas MT-Bench independently scores each response on a 1--10 scale. Appendix~\ref{app:judge_robustness} reports an additional evaluation using GLM-5.2 as a second judge, together with paired prompt-level bootstrap confidence intervals.

\paragraph{ME-Decoding Configuration.}
ME-Decoding selects a compact token subset by optimizing a Mahalanobis-style objective that balances token probability and embedding geometry. For Top-$W$ and ME-Decoding, both of which require an explicit candidate pool, we use the same top-$N$ pool with $N=512$. Except for the candidate-pool size, Top-$W$ uses its default hyperparameters. Top-$p$, Min-$p$, and Top-$H$ follow \citet{davoodi2026geometry}, and $p$-less is evaluated under its original parameter-free rule. ME-Decoding uses $\lambda=0.9$ by default, following Appendix~\ref{app:additional_experiments}. The official implementation is available at \href{https://github.com/sapphirexdy/ME_decoding}{https://github.com/sapphirexdy/ME\_decoding}.

\paragraph{Temperature Placement.}
At temperature $T$, every method operates on the same distribution $p_T=\operatorname{softmax}(z/T)$.
ME-Decoding uses $p_T$ for candidate scoring and samples from the renormalized selected support, while each baseline applies its truncation or reweighting rule to the same temperature-adjusted distribution.
\begin{figure*}[ht]
    \centering
    \includegraphics[width=\textwidth]{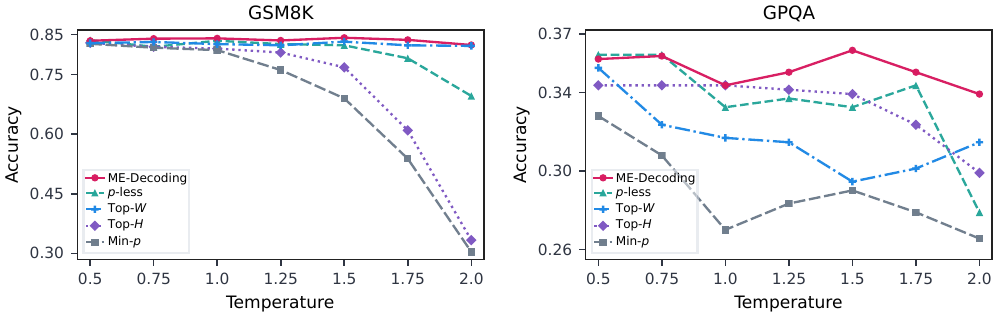}
    \caption{Accuracy vs. temperature curves of different decoding methods on GSM8K and GPQA.}
    \label{fig:acc_vs_temp}
\end{figure*}
\subsection{Main Results}
\label{sec:main_results}

\begin{figure*}[t]
    \centering
    \includegraphics[width=\textwidth]{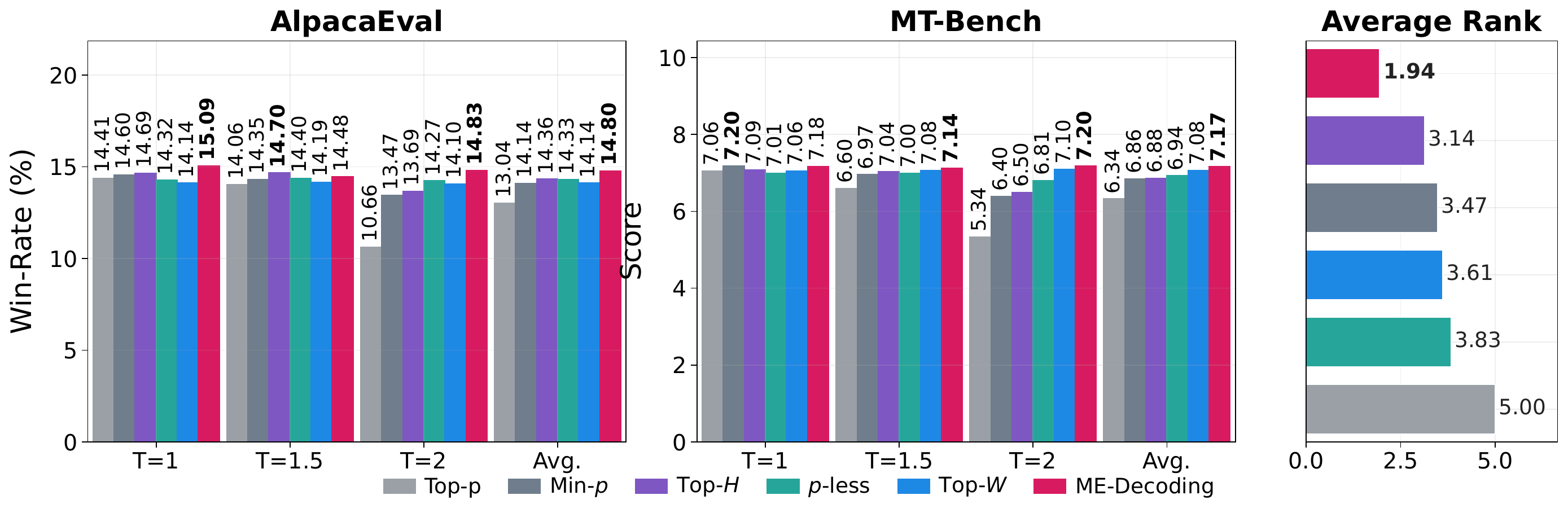}
    \caption{
    Open-ended generation performance and overall ranking comparison.
    The \textbf{left} and \textbf{middle} panels report AlpacaEval win rate and MT-Bench score across different temperatures, averaged over three models and aggregated over 3 runs.
    The \textbf{right} panel summarizes the average rank of each decoding method across all models, temperatures, and benchmarks, with bold numbers indicating the best overall rank.
    }
    \label{fig:instruction_chat_avg_rank}
\end{figure*}

\paragraph{Reasoning Benchmarks.}
We first evaluate ME-Decoding on GSM8K and GPQA across three models and three temperatures. The results are reported in Tables~\ref{tab:gsm8k_results} and~\ref{tab:gpqa_results}. Overall, ME-Decoding achieves the best average performance across both datasets.
Appendix~\ref{app:repeated_reasoning} reports three-seed repetitions as mean $\pm$ sample standard deviation and includes CraEG combined with $p$-less following its plug-in formulation.

Figure~\ref{fig:acc_vs_temp} further shows the accuracy--temperature curves of different decoding methods. Compared with existing baselines, ME-Decoding maintains a more stable accuracy profile as the temperature increases. Higher temperatures enlarge the sampling space, making probability-only truncation methods more prone to noisy tokens. By incorporating token geometry, ME-Decoding better preserves high-confidence candidates while filtering out uninformative ones, leading to a better balance between exploration and reasoning reliability.

\paragraph{Instruction-following and Chat.}

We further evaluate ME-Decoding on instruction-following and chat-style generation tasks using MT-Bench and AlpacaEval. Experiments are conducted on Qwen3-4B, Phi-4-mini, and Mistral-7B under $T\in\{1.0,1.5,2.0\}$. Figure~\ref{fig:instruction_chat_avg_rank} summarizes the results. The left and middle panels report the average AlpacaEval win rate and MT-Bench score over the three models at each temperature, while the right panel presents the overall average rank across all models, temperatures, and benchmarks. Detailed numerical results are provided in the Appendix.

As shown in Figure~\ref{fig:instruction_chat_avg_rank}, ME-Decoding achieves strong performance on both open-ended benchmarks. It obtains the best averaged AlpacaEval win rate and MT-Bench score. The overall rank comparison further shows that ME-Decoding attains the best average rank among all compared decoding methods, indicating its stable advantage across models, temperatures, and benchmarks. These results suggest that ME-Decoding is not limited to accuracy-oriented reasoning tasks, but also improves open-ended generation by selecting a compact and informative token set.
The same conclusion holds under a second automatic judge. Paired prompt-level bootstrap intervals against Top-$W$, $p$-less, and Top-$H$ are positive for both benchmarks and both judges, as reported in Appendix~\ref{app:judge_robustness}.

\begin{table*}[t]
\centering
\footnotesize
\setlength{\tabcolsep}{4.5pt}
\renewcommand{\arraystretch}{0.85}
\resizebox{\textwidth}{!}{
\begin{tabular}{llcccccc}
\toprule
Setting & Statistic
& Top-$p$ & Min-$p$ & $p$-less & Top-$H$ & Top-$W$
& \textbf{ME-Decoding} \\
\midrule
\multirow{2}{*}{\textbf{Large}}
& Mean s/token
& 0.01407 & 0.00196 & 0.00314 & 0.01493 & 0.13337 & 0.01799 \\
& Standard Deviation
& 0.00006 & 0.00002 & 0.00000 & 0.00009 & 0.00098 & 0.00560 \\
\midrule
\multirow{2}{*}{\textbf{Medium}}
& Mean s/token
& 0.00335 & 0.00054 & 0.00080 & 0.00330 & 0.00487 & 0.00179 \\
& Standard Deviation
& 0.00003 & 0.00001 & 0.00001 & 0.00005 & 0.00003 & 0.00010 \\
\bottomrule
\end{tabular}
}
\caption{
Average CPU sampling overhead per token on synthetic logits.
\textbf{Large}: vocabulary size 128K, embedding dimension 1024, top-$N=2048$.
\textbf{Medium}: vocabulary size 32K, embedding dimension 256, top-$N=512$.
}
\label{tab:sampling_time}
\end{table*}

\subsection{Analysis}

\paragraph{Complexity Analysis.}
We analyze the time complexity of ME-Decoding in Algorithm~\ref{alg:greedy_mes_decoding}. 
Let $N$ be the candidate-pool size, $d$ the embedding dimension, and $\tau$ the number of selected tokens before early stopping. 
The kernel computation costs $\mathcal{O}(N\tau d)$, and the greedy evaluation costs $\mathcal{O}(N\tau^3)$, leading to
$\mathcal{O}\bigl(N\tau(d+\tau^2)\bigr)$.
Since early stopping usually yields a small $\tau$ with $\tau \ll N$, the practical cost scales nearly linearly with $N$ and approaches $O(Nd)$ when $\tau$ is treated as a small constant.
A detailed analysis is provided in Appendix~\ref{app:complexity}.

\paragraph{Efficiency Analysis.}
To compare inference-time efficiency, we use a CPU-only synthetic logits benchmark that isolates sampling and logits-processing overhead.
We measure the average time for filtering logits and sampling one token under two settings: Medium with vocabulary size $32{,}000$, embedding dimension $256$, and top-$N=512$; and Large with vocabulary size $128{,}000$, embedding dimension $1024$, and top-$N=2048$.
All experiments use batch size $1$, one CPU thread, temperature $1.0$, $50$ warm-up steps, and $500$ measured steps over $10$ repeats.
Results are reported in Table~\ref{tab:sampling_time}.

As shown in Table~\ref{tab:sampling_time}, ME-Decoding introduces moderate overhead compared with purely probability-based methods due to its use of token embeddings, but remains highly efficient. It is about $2.7\times$ faster than Top-$W$ in the Medium setting and $7.4\times$ faster in the Large setting. In the Large setting, its overhead is also comparable to Top-$p$ and Top-$H$, demonstrating that ME-Decoding can exploit token-geometry information with substantially lower cost than existing geometry-aware decoding methods.
Appendix~\ref{app:gpu_efficiency} additionally reports end-to-end GPU latency and shows that model inference remains the dominant cost under the measured settings.

\paragraph{Diversity Analysis.}

We analyze the accuracy--diversity trade-off on GSM8K with Qwen3-4B. 
For each method, we randomly sample 500 questions and generate 10 responses per question. 
Accuracy is computed over all generations, and diversity is measured by Distinct-1/2 \citep{li2016diversity} and Self-BLEU \citep{zhu2018texygen}. 
We further average the min-max normalized Distinct-1, Distinct-2, and $1-\mathrm{Self\text{-}BLEU}$ within each temperature as an aggregated diversity score.

\begin{figure}[ht]
    \centering
    \includegraphics[width=\columnwidth]{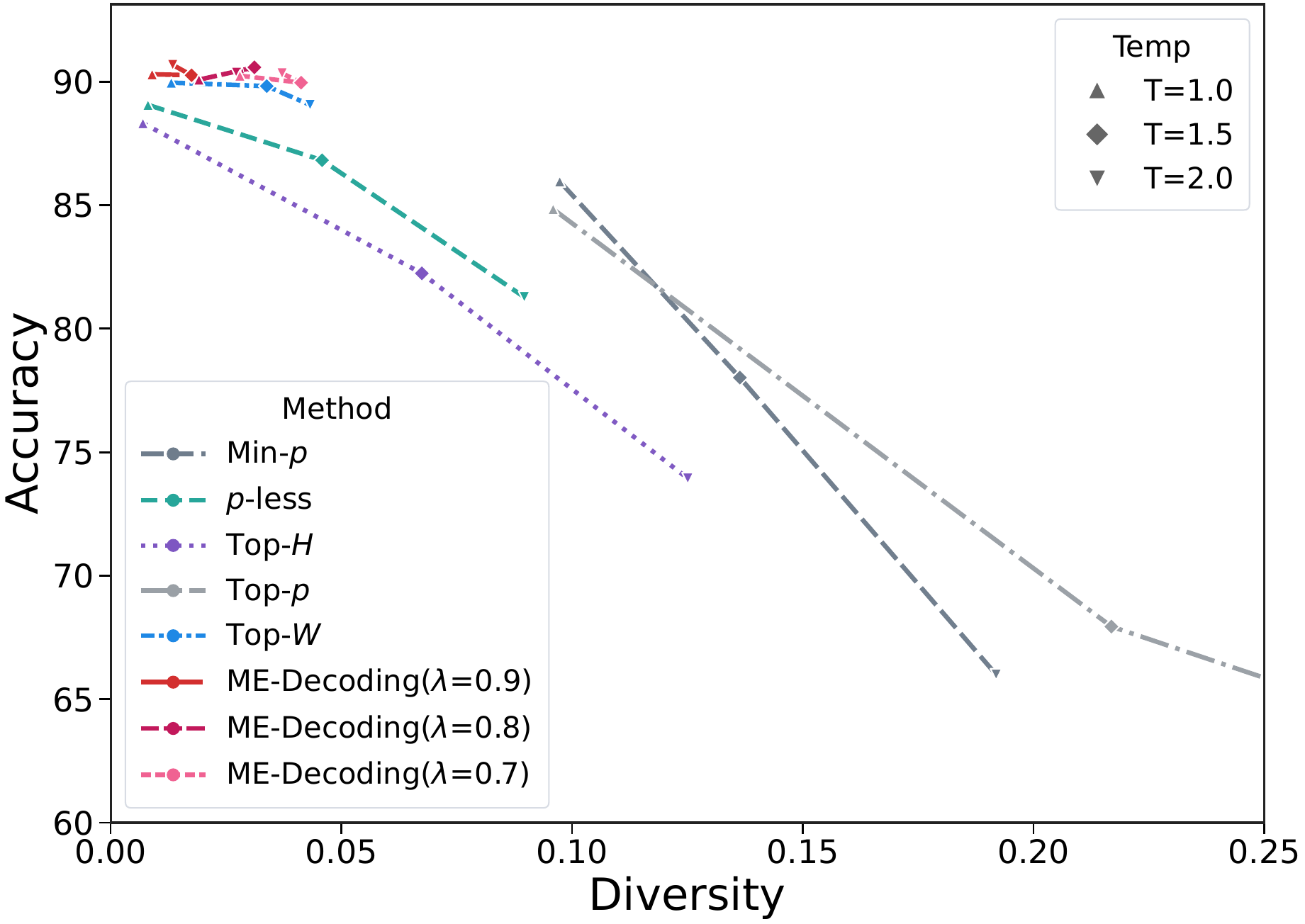}
    \caption{
    Accuracy--diversity trade-off on GSM8K using Qwen3-4B.
    Each method generates 10 responses for 500 sampled questions under each temperature.
    }
    \label{fig:diversity_vs_acc}
\end{figure}

Figure~\ref{fig:diversity_vs_acc} shows a clear accuracy--diversity trade-off.
More exploratory methods, such as Top-$p$, achieve higher diversity, especially at high temperatures, but suffer from notable accuracy degradation.
In contrast, ME-Decoding consistently achieves the highest accuracy and maintains stronger reasoning performance at comparable diversity levels.
This suggests that the proposed Mahalanobis-Ensemble objective improves token selection reliability by constructing a compact and informative sampling support, rather than simply increasing output randomness.
Detailed results are provided in Table~\ref{tab:diversity_temperature}.

To further test whether ME-Decoding constructs an informative geometry-aware support rather than merely truncating more aggressively, we compare it with probability-only controls calibrated on independent prompts to match its average support size, post-pruning entropy, or support-level pairwise cosine. At each step with $|S|\geq2$, Cos. is the mean cosine over all unordered pairs of L2-normalized selected-token embeddings; we then average it across eligible steps. Table~\ref{tab:support_diagnostics_main} reports results on Qwen2.5-1.5B at $T=1.0$ over three generation seeds.

\begin{table}[ht]
\centering
\footnotesize
\setlength{\tabcolsep}{2.4pt}
\resizebox{\columnwidth}{!}{
\begin{tabular}{llcccc}
\toprule
Dataset & Method & Acc. (\%) & Avg. $|S|$ & Cos. & $\MES$ \\
\midrule
\multirow{4}{*}{GSM8K}
& \textbf{ME-Decoding} & $\mathbf{63.20\pm0.57}$ & 1.076 & 0.151 & \textbf{0.756} \\
& Size-matched & $61.79\pm1.26$ & 1.067 & 0.210 & 0.707 \\
& Entropy-matched & $61.79\pm0.99$ & 1.070 & 0.208 & 0.708 \\
& Cosine-matched & $62.70\pm0.35$ & 1.014 & 0.198 & 0.715 \\
\midrule
\multirow{4}{*}{GPQA}
& \textbf{ME-Decoding} & $\mathbf{32.66\pm1.05}$ & 1.134 & 0.243 & \textbf{0.589} \\
& Size-matched & $28.47\pm1.17$ & 1.106 & 0.289 & 0.558 \\
& Entropy-matched & $28.47\pm1.17$ & 1.106 & 0.289 & 0.558 \\
& Cosine-matched & $29.25\pm0.91$ & 1.084 & 0.259 & 0.537 \\
\bottomrule
\end{tabular}}
\caption{ME-Decoding and approximately matched probability-only controls on Qwen2.5-1.5B at $T=1.0$. Accuracy is reported as mean $\pm$ sample standard deviation over three seeds.}
\label{tab:support_diagnostics_main}
\end{table}

Across both datasets, ME-Decoding achieves the highest accuracy and $\MES$ under approximately matched support size, entropy, or pairwise cosine. This result is consistent with our design goal of selecting a high-confidence support whose members provide complementary geometric information, and supports the effectiveness of directly optimizing $\MES$. Appendix~\ref{app:support_diagnostics} reports the complete diagnostics, including retained probability mass and entropy before pruning. Appendix~\ref{app:kernel_controls} further isolates the contribution of correctly aligned token geometry using identity- and permuted-kernel controls.


\section{Conclusion}
In this paper, we introduce ME-Decoding, a decoding framework that reformulates LLM token selection as ensemble pruning. 
ME-Decoding maximizes the \emph{Mahalanobis-Ensemble Score} ($\MES$) to select a compact token subset, combining token probabilities with an embedding-based similarity matrix to discount semantic redundancy. 
We further develop an efficient greedy algorithm with theoretical guarantees, including greedy-trajectory unimodality and an approximation bound. 
Experiments on reasoning and open-ended generation tasks show that ME-Decoding improves over representative probability-based and geometry-aware baselines with low inference overhead.

These results demonstrate the potential of ensemble pruning as a principled perspective for constructing compact, diverse, and reliable token candidate sets. Future work includes developing simpler and more effective objectives under this perspective, and extending ME-Decoding to broader scenarios such as long-form generation, multi-sample reasoning, and verification-augmented decoding.

\section*{Limitations}

Although ME-Decoding introduces an ensemble-pruning perspective for LLM decoding and achieves strong performance on several tasks, it still has several limitations. First, its performance depends on the construction of the similarity matrix, which controls the trade-off between token confidence and redundancy. While we provide a practical geometry-based kernel, more effective and adaptive similarity designs may further improve performance. The current kernel uses static token embeddings as a soft redundancy signal and may not capture all context-dependent semantic distinctions. Contextualized or hidden-state-conditioned kernels are therefore a useful direction for future work.
Second, although ME-Decoding consistently achieves higher accuracy than competing methods under comparable diversity levels, its output diversity is not always the highest among all decoding strategies. This suggests that the current kernel construction may still be conservative in promoting diverse generations. How to further improve generation diversity while preserving the accuracy gains of ME-Decoding remains an important direction for future work.
Third, determining the optimal subset size remains challenging, as in existing truncation-based decoding methods. The proposed $\MES$ criterion provides a principled selection rule, but it may not be optimal for all token distributions. Exploring adaptive subset-size rules and alternative objectives is another promising direction.

\section*{Ethical Considerations}

This work proposes ME-Decoding, a geometry-aware decoding framework that reformulates candidate token selection from the perspective of ensemble pruning. By selecting compact yet informative token sets, ME-Decoding aims to improve generation quality and reasoning performance without modifying model parameters or requiring additional training. A potential positive impact is that such plug-and-play inference-time methods may improve the usability of existing language models with limited computational overhead, thereby reducing deployment costs and the need for repeated model retraining.

At the same time, improved decoding and reasoning performance may also amplify downstream risks associated with generative systems. These risks include misinformation generation, hallucinated but plausible outputs, biased or harmful content, and the automation of low-quality or malicious text at scale. Since ME-Decoding operates at inference time and can be applied to a wide range of pretrained language models, its broader impact depends strongly on the deployment context, access control, and downstream use cases.

We encourage practitioners to use ME-Decoding together with established responsible deployment practices, including safety evaluation, bias and toxicity assessment, hallucination analysis, usage policies, and rate limits in open-ended settings. For high-stakes applications such as medical, legal, financial, or educational decision-making, model outputs should be carefully verified by qualified human experts. In addition, when ME-Decoding is applied to models trained on copyrighted, private, or sensitive data, developers should follow appropriate data governance, documentation, and licensing practices.

All datasets, models, and evaluation tools used in this work are publicly available. We follow their respective licenses and terms of use, and use them solely for research evaluation purposes. Our use of these artifacts is consistent with their intended use as benchmarks, pretrained models, and evaluation tools for research. We do not redistribute or repurpose any artifact beyond the scope allowed by its original access conditions, and we cite the original creators of all benchmarks, models, and baseline methods used in our experiments.

We do not collect any new user data or personally identifying information. The datasets used in our experiments are publicly available benchmarks released for research evaluation. We use them only under their standard evaluation settings and do not attempt to identify individuals or recover private information from any dataset.

\section*{Acknowledgments}

This work was supported by the Outstanding Innovative Talents Cultivation Funded Programs 2026 of Renmin University of China and the National Natural Science Foundation of China under Grant No.~12571301.



\bibliography{custom}

\input{appendix}

\end{document}

%% file: appendix.tex
\clearpage
\appendix

\onecolumn
\section{Proofs of Theoretical Results}

\label{app:proof}

\begin{proof}[\textbf{Proof of Theorem ~\ref{thm:mes_saturation}}]
Let $\MEE(S)=\boldsymbol{p}_S^\top\boldsymbol{K}_S^{-1}\boldsymbol{p}_S$ and $c_\lambda=1+\lambda H_2^T(p)$. Write $F_t=\MEE(S_t)$ and $\Delta_t=\MEE(S_{t+1})-\MEE(S_t)$. Then the greedy objective is defined as $\MES_t^g=\frac{F_t}{c_\lambda^t}$.

We first show that the restricted condition number controls the conditional residual correlations. Fix any subset $S\subseteq\mathcal{V}$ and two indices $i,k\notin S$. Define the conditional residual covariance matrix of $(i,k)$ given $S$ as the Schur complement:$$    \boldsymbol{C}_{\{i,k\}\mid S} = K_{\{i,k\}} - \boldsymbol{K}_{\{i,k\},S}\boldsymbol{K}_S^{-1}\boldsymbol{K}_{S,\{i,k\}}.$$Write$$    \boldsymbol{C}_{\{i,k\}\mid S} = 
    \begin{pmatrix}
        v_i & c_{ki\mid S}\\
        c_{ki\mid S} & v_k
    \end{pmatrix},$$where $v_i=K_{ii}-\boldsymbol{K}_{iS}\boldsymbol{K}_S^{-1}\boldsymbol{K}_{Si}$, $v_k=K_{kk}-\boldsymbol{K}_{kS}\boldsymbol{K}_S^{-1}\boldsymbol{K}_{Sk}$, and $c_{ki\mid S}=K_{ki}-\boldsymbol{K}_{kS}\boldsymbol{K}_S^{-1}\boldsymbol{K}_{Si}$. The corresponding conditional residual correlation is$$    \rho_{ki\mid S} = \frac{c_{ki\mid S}}{\sqrt{v_iv_k}}.$$Since $\boldsymbol{C}_{\{i,k\}\mid S}$ is the Schur complement of $\boldsymbol{K}_S$ in $\boldsymbol{K}_{S\cup\{i,k\}}$, its eigenvalues are bounded by those of $\boldsymbol{K}_{S\cup\{i,k\}}$. More precisely,$$     \lambda_{\min}(\boldsymbol{K}_{S\cup\{i,k\}}) \leq \lambda_{\max}(\boldsymbol{C}_{\{i,k\}\mid S}) \leq \lambda_{\max}(\boldsymbol{K}_{S\cup\{i,k\}}).$$These inequalities follow from the variational characterization of eigenvalues and the Schur-complement residualization argument. Therefore,$$    \kappa(\boldsymbol{C}_{\{i,k\}\mid S}) \leq \kappa(\boldsymbol{K}_{S\cup\{i,k\}}) \leq \kappa_N(\boldsymbol{K}).$$On the other hand, for any $2\times2$ positive definite matrix$$    \boldsymbol{C}=
    \begin{pmatrix}
        v_i & c\\
        c & v_k
    \end{pmatrix},$$with correlation $\rho=\frac{c}{\sqrt{v_iv_k}}$, we have:$$    |\rho|\leq \frac{\kappa(\boldsymbol{C})-1}{\kappa(\boldsymbol{C})+1},$$which gives$$    \kappa(\boldsymbol{C}) \geq \frac{1+|\rho|}{1-|\rho|}.$$Applying this to $\boldsymbol{C}_{\{i,k\}\mid S}$ gives$$    \frac{1+|\rho_{ki\mid S}|}{1-|\rho_{ki\mid S}|} \leq \kappa(\boldsymbol{C}_{\{i,k\}\mid S}) \leq \kappa_N(\boldsymbol{K}).$$By the assumption $\kappa_N(\boldsymbol{K})\leq c_\lambda$, we obtain\begin{equation}\frac{1+|\rho_{ki\mid S}|}{1-|\rho_{ki\mid S}|} \leq c_\lambda.\label{eq:rho_bound}\end{equation}Next, we prove the one-step growth control of greedy marginal gains. Fix a greedy step $t$, let $S=S_t$, and let $i$ be the element added at step $t$, i.e., $S_{t+1}=S_t\cup\{i\}$. For any remaining token $k\notin S\cup\{i\}$, define the standardized residual scores
    $$    z_i = \frac{p_{i}-\boldsymbol{K}_{iS}\boldsymbol{K}_S^{-1}\boldsymbol{p}_S}{\sqrt{K_{ii}-\boldsymbol{K}_{iS}\boldsymbol{K}_S^{-1}\boldsymbol{K}_{Si}}},$$
    and define $z_k$ analogously. By the Schur-complement formula, $\Delta \MEE(i\mid S)=z_i^2$ and $\Delta \MEE(k\mid S)=z_k^2$. Since $i$ is selected greedily, $z_k^2\leq z_i^2$. If $z_i=0$, then all remaining marginal gains are zero and the conclusion is immediate. Otherwise, define $r_k=\frac{z_k}{z_i}$, where $|r_k|\leq1$. After adding $i$, the marginal gain of $k$ is$$    \Delta \MEE(k\mid S\cup\{i\}) = \frac{(z_k-\rho_{ki\mid S}z_i)^2}{1-\rho_{ki\mid S}^2}.$$
    Hence

\begin{equation}
        \frac{\Delta \MEE(k\mid S\cup\{i\})}{\Delta \MEE(i\mid S)} 
    = \frac{(r_k-\rho_{ki\mid S})^2}{1-\rho_{ki\mid S}^2} 
    \leq \frac{(1+|\rho_{ki\mid S}|)^2}{1-\rho_{ki\mid S}^2} 
    = \frac{1+|\rho_{ki\mid S}|}{1-|\rho_{ki\mid S}|} 
    \leq c_\lambda,
\end{equation}

where the last inequality follows from \eqref{eq:rho_bound}. Taking the maximum over all remaining $k$ gives\begin{equation}\Delta_{t+1}\leq c_\lambda\Delta_t.\label{eq:delta_bound}\end{equation}Now suppose $\MES_{t+1}^g\leq \MES_t^g$. Since $\MES_t^g=\frac{F_t}{c_\lambda^t}$, this is equivalent to$$    \frac{F_{t+1}}{c_\lambda^{t+1}} \leq \frac{F_t}{c_\lambda^t}.$$Using $F_{t+1}=F_t+\Delta_t$, we get $F_t+\Delta_t\leq c_\lambda F_t$. Define $R_t=\frac{\Delta_t}{F_t}$. Then\begin{equation}\MES_{t+1}^g\leq \MES_{t}^g \quad\Longleftrightarrow\quad R_t\leq c_\lambda-1.\label{eq:rt_bound}\end{equation}Using \eqref{eq:delta_bound}, we have$$\begin{aligned}
    R_{t+1} 
    &= \frac{\Delta_{t+1}}{F_{t+1}} 
    = \frac{\Delta_{t+1}}{F_t+\Delta_t} \\
    &\leq \frac{c_\lambda\Delta_t}{F_t+\Delta_t} 
    = \frac{c_\lambda R_t}{1+R_t}.
\end{aligned}$$The function $h(R)=\frac{c_\lambda R}{1+R}$ is increasing for $R\geq0$. Therefore, if $R_t\leq c_\lambda-1$, then$$    R_{t+1} \leq h(R_t) \leq h(c_\lambda-1) = c_\lambda-1.$$By induction, $R_s\leq c_\lambda-1$ for all $s\geq t$. Using \eqref{eq:rt_bound}, we obtain$$    \MES_{s+1}^g\leq \MES_s^g, \qquad \forall s\geq t.$$Finally, define$$    \tau=\min\{t\in\{0,\dots,N-1\}:\MES_{t+1}^g\leq \MES_t^g\},$$
the sequence increases strictly before $\tau$ and is nonincreasing after $\tau$. Hence$$    \MES_{\tau}^g = \max_{0\leq r\leq N}\MES_r^g.$$This completes the proof.\end{proof}

\begin{lemma}[Approximation Guarantee for Greedy Selection]
\label{lemma1}
Let $f: 2^U \to \mathbb{R}_{\ge 0}$ be a normalized ($f(\emptyset)=0$), nonnegative, monotone set function, and let $\mathrm{OPT} = \max_{|S| \le k} f(S)$ denote the maximum value obtained by any set of size at most $k$. Let $S$ be the set selected by the Greedy algorithm. Then, the solution satisfies the following approximation guarantee:
\begin{equation}
    f(S) \geq \left( 1 - e^{-\gamma_{U,k}} \right) \cdot \mathrm{OPT},
\end{equation}
where $\gamma_{U,k}$ is the submodularity ratio of $f$. Formally, $\gamma_{U,k}$ is defined as the minimum ratio of the marginal gain of a set to the marginal gain of its individual elements:
\begin{equation}
    \gamma_{U,k} = \min_{L \subseteq U, \, A: |A| \le k, \, L \cap A = \emptyset} \frac{\sum_{x \in A} \left( f(L \cup \{x\}) - f(L) \right)}{f(L \cup A) - f(L)}.
\end{equation}
The ratio is taken over pairs $(L,A)$ such that $f(L\cup A)>f(L)$; if the denominator is zero, the ratio is defined as $1$.
\end{lemma}

\begin{proof}
This is the standard approximation guarantee for greedy maximization of a nonnegative monotone set function with submodularity ratio $\gamma_{U,k}$. Proof can be found in \citet{das2018approximate}.
\end{proof}

\noindent While Lemma~\ref{lemma1} offers a general guarantee, the ratio $\gamma_{\mathcal V,k}$ is typically intractable. We further provide a concrete bound for our Mahalanobis-Ensemble Energy ($\MEE$) by considering the specific objective
\[
    f(S)=\MEE(S)=\boldsymbol{p}_S^\top\boldsymbol{K}_S^{-1}\boldsymbol{p}_S.
\]
The following lemma bounds $\gamma_{\mathcal V,k}$ via the restricted minimum eigenvalue of the kernel matrix $\boldsymbol{K}$.

\begin{lemma}[Schur Complement Preserves the Minimum Eigenvalue]
\label{lem:A2}
Let $\boldsymbol{K}\in\mathbb R^{n\times n}$ be a symmetric positive definite matrix. Write

$$\boldsymbol{K}=
\begin{pmatrix}
\boldsymbol{K}_{11} & \boldsymbol{s}\\
\boldsymbol{s}^\top & K_{nn}
\end{pmatrix},
\quad
\boldsymbol{K}_{11}\in\mathbb R^{(n-1)\times(n-1)},\;
\boldsymbol{s}\in\mathbb R^{(n-1)\times 1},\;
K_{nn}>0.$$

Define the Schur complement

$$\boldsymbol{K}' = \boldsymbol{K}_{11}-\frac{1}{K_{nn}}\boldsymbol{s}\boldsymbol{s}^\top.$$

Then

$$\lambda_{\min}(\boldsymbol{K})\leq \lambda_{\min}(\boldsymbol{K}').$$

\end{lemma}

\begin{proof}
Let $\lambda_1'=\lambda_{\min}(\boldsymbol{K}')$, and choose a nonzero eigenvector $\boldsymbol{e}'\in\mathbb R^{n-1}$ such that $\boldsymbol{K}'\boldsymbol{e}'=\lambda_1'\boldsymbol{e}'$. We construct $\boldsymbol{e}=\begin{pmatrix} \boldsymbol{e}' \\ e_n \end{pmatrix}$, where $e_n=-\frac{1}{K_{nn}}\boldsymbol{s}^\top \boldsymbol{e}'$. Using the block form of $\boldsymbol{K}$, we have:

$$\boldsymbol{K}\boldsymbol{e} = \begin{pmatrix} \boldsymbol{K}_{11}\boldsymbol{e}'+\boldsymbol{s}e_n \\ \boldsymbol{s}^\top \boldsymbol{e}'+K_{nn}e_n \end{pmatrix}.$$

By the definition of $e_n$, the lower block becomes $\boldsymbol{s}^\top \boldsymbol{e}'+K_{nn}e_n=0$. \
For the upper block, we have $\boldsymbol{K}_{11}\boldsymbol{e}'+\boldsymbol{s}e_n = \boldsymbol{K}_{11}\boldsymbol{e}'-\frac{1}{K_{nn}}\boldsymbol{s}\boldsymbol{s}^\top \boldsymbol{e}' = \left(\boldsymbol{K}_{11}-\frac{1}{K_{nn}}\boldsymbol{s}\boldsymbol{s}^\top\right)\boldsymbol{e}' = \boldsymbol{K}'\boldsymbol{e}' = \lambda_1'\boldsymbol{e}'$.
Therefore, $\boldsymbol{K}\boldsymbol{e} = \begin{pmatrix} \lambda_1'\boldsymbol{e}' \\ 0 \end{pmatrix}$.

By the Rayleigh quotient, we know $\lambda_{\min}(\boldsymbol{K}) = \min_{\boldsymbol{x}\neq 0}\frac{\boldsymbol{x}^\top \boldsymbol{K}\boldsymbol{x}}{\boldsymbol{x}^\top \boldsymbol{x}} \leq \frac{\boldsymbol{e}^\top \boldsymbol{K}\boldsymbol{e}}{\boldsymbol{e}^\top \boldsymbol{e}}$. \
Since $\boldsymbol{e}^\top \boldsymbol{K}\boldsymbol{e}=\lambda_1'\|\boldsymbol{e}'\|_2^2$ and $\boldsymbol{e}^\top \boldsymbol{e}=\|\boldsymbol{e}'\|_2^2+e_n^2\geq \|\boldsymbol{e}'\|_2^2$, we can bound the quotient as follows:

$$\lambda_{\min}(\boldsymbol{K}) \leq \frac{\boldsymbol{e}^\top \boldsymbol{K}\boldsymbol{e}}{\boldsymbol{e}^\top \boldsymbol{e}} = \frac{\lambda_1'\|\boldsymbol{e}'\|_2^2}{\|\boldsymbol{e}'\|_2^2+e_n^2} \leq \lambda_1'.$$

Since $\boldsymbol{K}'\succ 0$, we have $\lambda_1'>0$, which implies $\lambda_1' = \lambda_{\min}(\boldsymbol{K}')$. This proves the claim.
\end{proof}

\begin{lemma}[Spectral Bound on Submodularity Ratio]
\label{lem:A3}
Assume that $\boldsymbol{K}\succ0$ is a correlation matrix. For $f(S)=\boldsymbol{p}_S^\top\boldsymbol{K}_S^{-1} \boldsymbol{p}_S$ and disjoint sets $L, A$, let $\boldsymbol{K}_{L \cup A}$ be partitioned naturally into blocks $\boldsymbol{K}_L, \boldsymbol{K}_A, \boldsymbol{K}_{LA}, \boldsymbol{K}_{AL}$. We define the residual statistics $\tilde{\boldsymbol{p}}$ and $\tilde{\boldsymbol{K}}$ (Schur complement) as:

$$    \tilde{\boldsymbol{p}} = \boldsymbol{p}_A-\boldsymbol{K}_{A L} \boldsymbol{K}_L^{-1} \boldsymbol{p}_L, \quad 
    \tilde{\boldsymbol{K}} = \boldsymbol{K}_A-\boldsymbol{K}_{A L} \boldsymbol{K}_L^{-1} \boldsymbol{K}_{L A}.$$

The submodularity ratio $\gamma_{\mathcal{V}, k}$, which involves minimizing the term $\frac{\tilde{\boldsymbol{p}}^\top[\operatorname{diag}(\tilde{\boldsymbol{K}})]^{-1} \tilde{\boldsymbol{p}}}{\tilde{\boldsymbol{p}}^\top \tilde{\boldsymbol{K}}^{-1} \tilde{\boldsymbol{p}}}$, is bounded from below by the eigenvalues of $\boldsymbol{K}$:

$$        \gamma_{\mathcal{V}, k} \geq \min_{\substack{L\subseteq\mathcal V,\; A\subseteq\mathcal V\setminus L\\ |A|\le k}} \lambda_{\min }(\boldsymbol{K}, |L \cup A|) \geq \lambda_{\min }(\boldsymbol{K}, N),$$

where $\lambda_{\min}(\boldsymbol{K}, k) \triangleq \min _{S:|S|=k} \lambda_{\min}\left(\boldsymbol{K}_S\right)$ denotes the smallest eigenvalue among all $k \times k$ principal submatrices of $\boldsymbol{K}$, and $N = |\mathcal{V}|$.
\end{lemma}

\begin{proof}
Consider disjoint $L$ and $A$. Using the block inverse formula for $\boldsymbol{K}_{L\cup A}^{-1}$, one obtains the decomposition

$$f(L\cup A)=\boldsymbol{p}_L^\top \boldsymbol{K}_L^{-1} \boldsymbol{p}_L \;+\;\tilde{\boldsymbol{p}}^\top \tilde{\boldsymbol{K}}^{-1}\tilde{\boldsymbol{p}},$$

hence

$$f(L\cup A)-f(L)=\tilde{\boldsymbol{p}}^\top \tilde{\boldsymbol{K}}^{-1}\tilde{\boldsymbol{p}}.$$

For a singleton $a\in A$, the same calculation with $|A|=1$ gives

$$f(L\cup\{a\})-f(L)=\frac{\tilde{p}_a^2}{\tilde{K}_{aa}}.$$

Therefore, for any $L\subseteq \mathcal{V}$ and $A$ with $|A|\le k$ and $A\cap L=\emptyset$, the ratio appearing in the definition of the submodularity ratio becomes

$$\frac{\sum_{a\in A}\big(f(L\cup\{a\})-f(L)\big)}{f(L\cup A)-f(L)}
=\frac{\tilde{\boldsymbol{p}}^\top\big[\mathrm{diag}(\tilde{\boldsymbol{K}})\big]^{-1}\tilde{\boldsymbol{p}}}
{\tilde{\boldsymbol{p}}^\top \tilde{\boldsymbol{K}}^{-1}\tilde{\boldsymbol{p}}}.$$

Then, let $\boldsymbol{D}=\mathrm{diag}(\tilde{\boldsymbol{K}})$ and define the correlation matrix $\tilde{\boldsymbol{K}}_{\rho}=\boldsymbol{D}^{-1/2}\tilde{\boldsymbol{K}}\boldsymbol{D}^{-1/2}$. Let $\boldsymbol{v}=\boldsymbol{D}^{-1/2}\tilde{\boldsymbol{p}}$. Then

$$\tilde{\boldsymbol{p}}^\top\boldsymbol{D}^{-1}\tilde{\boldsymbol{p}}=\|\boldsymbol{v}\|_2^2,\qquad
\tilde{\boldsymbol{p}}^\top \tilde{\boldsymbol{K}}^{-1}\tilde{\boldsymbol{p}} = \boldsymbol{v}^\top \tilde{\boldsymbol{K}}_{\rho}^{-1} \boldsymbol{v}.$$

Hence the ratio equals $\frac{\boldsymbol{v}^\top \boldsymbol{v}}{\boldsymbol{v}^\top \tilde{\boldsymbol{K}}_{\rho}^{-1}\boldsymbol{v}}$. By Rayleigh--Ritz, $\boldsymbol{v}^\top \tilde{\boldsymbol{K}}_{\rho}^{-1}\boldsymbol{v} \le \lambda_{\max}(\tilde{\boldsymbol{K}}_{\rho}^{-1})\,\boldsymbol{v}^\top \boldsymbol{v} =\frac{1}{\lambda_{\min}(\tilde{\boldsymbol{K}}_{\rho})}\boldsymbol{v}^\top \boldsymbol{v}$, so

$$\frac{\tilde{\boldsymbol{p}}^\top\boldsymbol{D}^{-1}\tilde{\boldsymbol{p}}}{\tilde{\boldsymbol{p}}^\top \tilde{\boldsymbol{K}}^{-1}\tilde{\boldsymbol{p}}} \ge \lambda_{\min}(\tilde{\boldsymbol{K}}_{\rho}).$$

We can eliminate the elements of $L$ one by one via residualization; each elimination replaces the current covariance by the covariance of residuals. By Lemma~\ref{lem:A2}, each such residualization step cannot decrease the smallest eigenvalue. After eliminating all of $L$, we obtain the Schur complement $\tilde{\boldsymbol{K}}$ corresponding to conditioning on $L$, and thus

$$\lambda_{\min}(\boldsymbol{K}_{L\cup A}) \le \lambda_{\min}(\tilde{\boldsymbol{K}}).$$

Finally, normalizing $\tilde{\boldsymbol{K}}$ to unit variance gives $\tilde{\boldsymbol{K}}_{\rho}$, and $\lambda_{\min}(\tilde{\boldsymbol{K}})\le \lambda_{\min}(\tilde{\boldsymbol{K}}_{\rho})$. Combining the last two displays:

$$\lambda_{\min}(\tilde{\boldsymbol{K}}_{\rho}) \ge \lambda_{\min}(\boldsymbol{K}_{L\cup A}) \ge \lambda_{\min}(\boldsymbol{K},|L\cup A|).$$

By definition of $\gamma_{\mathcal{V},k}$, we conclude

$$\gamma_{\mathcal{V},k} \ge \min_{\substack{L\subseteq\mathcal V,\; A\subseteq\mathcal V\setminus L\\ |A|\le k}} \lambda_{\min}(\boldsymbol{K},|L\cup A|).$$

Since $\boldsymbol{K}_{L\cup A}$ is a principal submatrix of $\boldsymbol{K}$, Cauchy's interlacing theorem gives

$$\lambda_{\min}(\boldsymbol{K}_{L\cup A}) \ge \lambda_{\min}(\boldsymbol{K}).$$

In our notation, $\lambda_{\min}(\boldsymbol{K})=\lambda_{\min}(\boldsymbol{K},N)$, which naturally completes the proof.
\end{proof}

\begin{lemma}[Approximation Bound for Greedy $\MEE$]
\label{lemma:mee_greedy_bound_final}
Assume that $\boldsymbol{K}\succ0$ is a correlation matrix defined on ground set $\mathcal{V}$ ($|\mathcal{V}|=N$). Let $S_k$ be the greedy set after $k$ steps. Then the approximation bound holds:

$$    f(S_k) \geq \left(1-\exp\{-\lambda_{\min}(\boldsymbol{K},N)\}\right) \max_{\substack{S\subseteq\mathcal{V}\\ |S|\leq k}} f(S).$$

\end{lemma}

\begin{proof}
Let $\mathrm{OPT} \triangleq \max_{S\subseteq \mathcal{V},\,|S|\le k} \MEE(S)$.

First, since the Mahalanobis-Ensemble Energy objective $\MEE(S) = \boldsymbol{p}_S^\top\boldsymbol{K}_S^{-1}\boldsymbol{p}_S$ is a normalized, nonnegative, and monotone set function, we can apply the general greedy guarantee from Lemma~\ref{lemma1}. This yields:

$$\MEE(S_k) \ge \bigl(1-e^{-\gamma_{\mathcal{V},k}}\bigr)\cdot \mathrm{OPT},$$

where $\gamma_{\mathcal{V},k}$ is the submodularity ratio of $\MEE$ on the ground set $\mathcal{V}$.

Next, we lower bound the submodularity ratio $\gamma_{\mathcal{V},k}$. By Lemma~\ref{lem:A3}, the submodularity ratio satisfies:

$$\gamma_{\mathcal V,k} \ge \min_{\substack{L\subseteq\mathcal V,\; A\subseteq\mathcal V\setminus L\\ |A|\le k}} \lambda_{\min}(\boldsymbol{K},|L\cup A|).$$

Since $L \cup A \subseteq \mathcal{V}$, we have $|L \cup A| \le N$. By the interlacing theorem, every principal submatrix of $\boldsymbol{K}$ has a minimum eigenvalue at least $\lambda_{\min}(\boldsymbol{K},N)$. Hence:

$$\gamma_{\mathcal V,k} \ge \lambda_{\min}(\boldsymbol{K},N).$$

Substituting this spectral lower bound back into the greedy approximation guarantee gives:

$$\MEE(S_k) \ge \bigl(1-e^{-\lambda_{\min}(\boldsymbol{K},N)}\bigr)\cdot \mathrm{OPT}.$$

This completes the proof.
\end{proof}

\begin{proof}[\textbf{Proof of Theorem ~\ref{thm:mes_stopping_approx}}]
Let $\MEE(S)=\boldsymbol{p}_S^\top\boldsymbol{K}_S^{-1}\boldsymbol{p}_S$ and $c_\lambda=1+\lambda H_2^T(\boldsymbol{p})$. The objective can be written as $\MES_{\lambda}(S) = \frac{\MEE(S)}{c_\lambda^{|S|}}$. Define the greedy approximation factor as $\alpha_N = 1-\exp\{-\lambda_{\min}(\boldsymbol{K},N)\}$.

By Lemma~\ref{lemma:mee_greedy_bound_final}, for every $k=1,\ldots,N$, the greedy set $S_k$ satisfies:
\begin{equation}
\MEE(S_k) \geq \alpha_N \max_{\substack{S\subseteq\mathcal{V}\ |S|\leq k}} \MEE(S).
\label{eq:mee_greedy_step}
\end{equation}
Fix any $k\in\{1,\ldots,N\}$. Dividing both sides of \eqref{eq:mee_greedy_step} by $c_\lambda^k$ gives:

$$\begin{aligned}
    \MES_{\lambda}(S_k) 
    = \frac{\MEE(S_k)}{c_\lambda^k}
    &\geq \alpha_N \frac{\max_{S\subseteq\mathcal{V}, |S|\leq k} \MEE(S)}{c_\lambda^k} \\
    &\geq \alpha_N \frac{\max_{S\subseteq\mathcal{V}, |S|=k} \MEE(S)}{c_\lambda^k} \\
    &= \alpha_N \max_{\substack{S\subseteq\mathcal{V}\\ |S|=k}} \frac{\MEE(S)}{c_\lambda^{|S|}} 
    = \alpha_N \max_{\substack{S\subseteq\mathcal{V}\\ |S|=k}} \MES_{\lambda}(S).
\end{aligned}$$

Taking the maximum over all subset sizes $k=1,\ldots,N$ on both sides, we obtain the global bound for the greedy trajectory:
\begin{equation}
\max_{1\leq k\leq N} \MES_{\lambda}(S_k)
\geq \alpha_N \max_{1\leq k\leq N} \max_{\substack{S\subseteq\mathcal{V}\ |S|=k}} \MES_{\lambda}(S).
\label{eq:mes_global_max}
\end{equation}

By Theorem~\ref{thm:mes_saturation}, the greedy saturation stopping point $\tau$ achieves the exact maximum of the entire greedy sequence, meaning $\MES_{\tau}^{g} = \max_{0\leq r\leq N}\MES_{r}^{g}$. Since $\MES_r^g=\MES_{\lambda}(S_r)$, we have:

$$    \MES_{\tau}^{g} \geq \max_{1\leq k\leq N} \MES_{\lambda}(S_k).$$

Combining this with \eqref{eq:mes_global_max} and substituting the definition of $\alpha_N$, we arrive at our final conclusion:

$$    \MES_{\tau}^{g} \geq \left(1-\exp\{-\lambda_{\min}(\boldsymbol{K},N)\}\right) \max_{\emptyset\neq S\subseteq\mathcal{V}} \MES_{\lambda}(S).$$

This completes the proof.
\end{proof}

\clearpage
\twocolumn
\section{Additional Experiments}
\label{app:additional_experiments}

This section provides repeated evaluations, controlled ablations, diagnostic analyses, and efficiency measurements.
All experiments in this work were conducted on a single NVIDIA GeForce RTX 4090 GPU.
Unless stated otherwise, repeated generation results use three shared seeds and report the mean $\pm$ sample standard deviation.

\subsection{Additional Results for Reasoning Benchmarks}
\label{app:repeated_reasoning}

Tables~\ref{tab:gsm8k_repeated} and~\ref{tab:gpqa_repeated} repeat the reasoning evaluation over three seeds and report CraEG combined with $p$-less.
This composition follows CraEG's plug-in role as a post-softmax reweighting module rather than a standalone support-selection rule, while the tables use \emph{CraEG} as the concise method name.

\begin{table*}[t]
\centering
\setlength{\tabcolsep}{2.5pt}
\resizebox{\textwidth}{!}{
\begin{tabular}{lccc|ccc|ccc|c}
\toprule
& \multicolumn{3}{c|}{Qwen3-4B-Instruct}
& \multicolumn{3}{c|}{Phi-4-mini-Instruct}
& \multicolumn{3}{c|}{Mistral-7B-Instruct}
& \\
\cmidrule(lr){2-4}\cmidrule(lr){5-7}\cmidrule(lr){8-10}
Method & $T=1.0$ & $T=1.5$ & $T=2.0$
& $T=1.0$ & $T=1.5$ & $T=2.0$
& $T=1.0$ & $T=1.5$ & $T=2.0$ & Avg. \\
\midrule
$p$-less
& $78.82_{\pm 0.24}$ & $72.53_{\pm 0.50}$ & $64.47_{\pm 1.00}$
& $82.99_{\pm 0.76}$ & $82.99_{\pm 0.39}$ & $71.14_{\pm 0.99}$
& $53.53_{\pm 0.73}$ & $51.96_{\pm 0.74}$ & $46.07_{\pm 1.30}$ & 67.17 \\
Top-$W$
& $77.63_{\pm 0.35}$ & $76.65_{\pm 0.23}$ & $75.54_{\pm 0.72}$
& $83.24_{\pm 0.13}$ & $83.52_{\pm 0.23}$ & $82.99_{\pm 0.18}$
& $53.07_{\pm 0.40}$ & $53.50_{\pm 1.05}$ & $52.44_{\pm 1.26}$ & 70.95 \\
CraEG
& $72.23_{\pm 1.71}$ & $65.07_{\pm 1.00}$ & $58.12_{\pm 1.32}$
& $83.52_{\pm 0.55}$ & $81.91_{\pm 1.07}$ & $65.38_{\pm 1.26}$
& $52.72_{\pm 0.12}$ & $50.82_{\pm 0.57}$ & $40.56_{\pm 1.64}$ & 63.37 \\
\rowcolor{gray!18}\textbf{Ours}
& $\mathbf{80.04}_{\pm 0.69}$ & $\mathbf{79.56}_{\pm 0.35}$ & $\mathbf{80.04}_{\pm 0.80}$
& $\mathbf{83.90}_{\pm 0.24}$ & $\mathbf{83.95}_{\pm 0.31}$ & $82.44_{\pm 0.12}$
& $\mathbf{54.18}_{\pm 1.03}$ & $\mathbf{53.55}_{\pm 0.12}$ & $\mathbf{53.98}_{\pm 0.57}$ & \textbf{72.40} \\
\bottomrule
\end{tabular}}
\caption{GSM8K flexible-extract accuracy (\%) over three generation seeds. Subscripts report sample standard deviations, and the Avg. column averages the nine model--temperature means.}
\label{tab:gsm8k_repeated}
\end{table*}

\begin{table*}[t]
\centering
\setlength{\tabcolsep}{2.5pt}
\resizebox{\textwidth}{!}{
\begin{tabular}{lccc|ccc|ccc|c}
\toprule
& \multicolumn{3}{c|}{Qwen3-4B-Instruct}
& \multicolumn{3}{c|}{Phi-4-mini-Instruct}
& \multicolumn{3}{c|}{Mistral-7B-Instruct}
& \\
\cmidrule(lr){2-4}\cmidrule(lr){5-7}\cmidrule(lr){8-10}
Method & $T=1.0$ & $T=1.5$ & $T=2.0$
& $T=1.0$ & $T=1.5$ & $T=2.0$
& $T=1.0$ & $T=1.5$ & $T=2.0$ & Avg. \\
\midrule
$p$-less
& $35.57_{\pm 2.29}$ & $35.86_{\pm 1.57}$ & $33.85_{\pm 1.35}$
& $32.81_{\pm 1.77}$ & $32.66_{\pm 0.78}$ & $30.06_{\pm 2.79}$
& $28.42_{\pm 1.44}$ & $28.72_{\pm 1.23}$ & $27.68_{\pm 1.39}$ & 31.74 \\
Top-$W$
& $35.49_{\pm 0.45}$ & $34.90_{\pm 0.78}$ & $35.86_{\pm 0.26}$
& $31.62_{\pm 0.78}$ & $30.58_{\pm 0.77}$ & $30.65_{\pm 1.49}$
& $28.79_{\pm 1.02}$ & $29.39_{\pm 0.34}$ & $29.39_{\pm 1.49}$ & 31.85 \\
CraEG
& $35.71_{\pm 1.24}$ & $34.08_{\pm 2.17}$ & $33.93_{\pm 0.97}$
& $32.22_{\pm 1.42}$ & $32.14_{\pm 0.80}$ & $29.54_{\pm 0.68}$
& $28.65_{\pm 0.68}$ & $28.72_{\pm 2.03}$ & $28.20_{\pm 1.23}$ & 31.46 \\
\rowcolor{gray!18}\textbf{Ours}
& $\mathbf{36.68}_{\pm 0.13}$ & $\mathbf{36.76}_{\pm 0.68}$ & $\mathbf{36.61}_{\pm 0.59}$
& $\mathbf{34.30}_{\pm 0.68}$ & $\mathbf{34.45}_{\pm 0.78}$ & $\mathbf{34.08}_{\pm 0.56}$
& $\mathbf{29.09}_{\pm 1.01}$ & $\mathbf{29.84}_{\pm 0.34}$ & $29.32_{\pm 0.46}$ & \textbf{33.46} \\
\bottomrule
\end{tabular}}
\caption{GPQA accuracy (\%) over three generation seeds. Subscripts report sample standard deviations, and the Avg. column averages the nine model--temperature means.}
\label{tab:gpqa_repeated}
\end{table*}

Across the nine model--temperature settings, ME-Decoding achieves the highest average accuracy on both GSM8K and GPQA.
It leads in eight of the nine settings on each benchmark, supporting stable improvements across seeds and temperatures.

\subsection{Ablation Studies}
\label{app:kernel_controls}

\paragraph{Probability and Kernel Controls.}

We use Greedy decoding to test whether repeatedly selecting the locally most probable token is sufficient.
At $T=1.0$, ME-Decoding improves GSM8K accuracy from $61.41\%$ to $63.20\pm0.57\%$ and GPQA accuracy from $32.37\%$ to $32.66\pm1.05\%$ on Qwen2.5-1.5B.
Greedy is deterministic, whereas ME-Decoding results average three generation seeds.

We then compare the semantic kernel with two controlled alternatives.
The identity kernel $K=I$ removes pairwise geometry, reducing the MEE numerator to $\sum_{i\in S}p_i^2$.
The permuted kernel $K_{\mathrm{perm}}=P K_{\mathrm{true}}P^\top$ preserves the eigenvalues, condition number, entry multiset, and numerical scale of the semantic kernel while disrupting its alignment with token identities.
All variants use the same prompts, probabilities, generation seeds, and pruning hyperparameters.

\begin{center}
\begin{minipage}{0.96\columnwidth}
\centering
\setlength{\tabcolsep}{6pt}
\begin{tabular}{lccc}
\toprule
Dataset & True $K$ & Permuted $K$ & $K=I$ \\
\midrule
GSM8K & \textbf{62.32} & 61.15 & 60.85 \\
GPQA & \textbf{32.64} & 32.09 & 29.81 \\
\bottomrule
\end{tabular}
\captionof{table}{Accuracy (\%) averaged over $T\in\{1.0,1.5,2.0\}$ in the Qwen2.5-1.5B kernel-control experiment. Both controls reduce accuracy relative to the correctly aligned semantic kernel.}
\label{tab:kernel_controls_average}
\end{minipage}
\end{center}

The correctly aligned semantic kernel achieves the highest average accuracy on both datasets, while removing or permuting the geometry reduces performance.

\paragraph{Component Ablations.}
We examine the contribution of the adaptive bandwidth $\epsilon$ and the kernel matrix $K$.
Removing the adaptive bandwidth degrades performance, especially at higher temperatures, which indicates that an appropriately scaled kernel is important for robust optimization.
Replacing $K$ with the identity matrix removes the embedding-based similarity structure and also reduces accuracy.

\begin{table}[htbp]
\centering
\footnotesize
\begin{tabular}{lcccc}
\toprule
Dataset
& $T$
& ME-Decoding
& Without $\epsilon$
& $K=I$ \\
\midrule
\multirow{3}{*}{GSM8K}
& 1.0 & 84.08 & $-2.20$ & $-0.15$ \\
& 1.5 & 84.23 & $-4.62$ & $-0.83$ \\
& 2.0 & 82.41 & $-8.19$ & $-1.36$ \\
\midrule
\multirow{3}{*}{GPQA}
& 1.0 & 34.38 & $-0.22$ & $-0.67$ \\
& 1.5 & 36.16 & $-1.56$ & $-2.90$ \\
& 2.0 & 33.93 & $-0.22$ & $-0.89$ \\
\bottomrule
\end{tabular}
\caption{Component ablations under different temperatures. The ME-Decoding column reports absolute accuracy, while the remaining columns report changes relative to ME-Decoding.}
\label{tab:ablation_results}
\end{table}
\FloatBarrier

\subsection{Empirical Validation of Theorem~\ref{thm:mes_saturation}}
\label{app:unimodality_probe}

The condition in Theorem~\ref{thm:mes_saturation} is sufficient rather than necessary for trajectory unimodality.
To examine whether the predicted behavior occurs in practice, we disable early stopping and record the complete greedy $\MES$ trajectory for a 512-trajectory probe from each benchmark using Qwen2.5-1.5B at $T=1.0$.

\begin{center}
\begin{minipage}{0.96\columnwidth}
\centering
\setlength{\tabcolsep}{7pt}
\begin{tabular}{lcc}
\toprule
Dataset & Unimodal Trajectories & Rate \\
\midrule
GSM8K & 512/512 & \textbf{100\%} \\
GPQA & 512/512 & \textbf{100\%} \\
\bottomrule
\end{tabular}
\captionof{table}{Empirical unimodality of complete greedy $\MES$ trajectories.}
\label{tab:unimodality_probe}
\end{minipage}
\end{center}

The observation supports the practical relevance of the stopping behavior without replacing the theorem's sufficient-condition analysis.
The approximation result has a separate scope and depends on restricted minimum eigenvalues of selected kernel submatrices.

\subsection{Efficiency and Complexity}
\label{app:gpu_efficiency}

We benchmark 200 GSM8K problems, generating 64 tokens per problem with $T=1.5$, batch size $1$, and $N=512$.
Each method therefore generates 12,800 tokens.
ME-Decoding uses a custom Triton kernel, and all methods use identical filtering and sampling infrastructure.

\begin{table*}[t]
\centering
\setlength{\tabcolsep}{4pt}
\resizebox{\textwidth}{!}{
\begin{tabular}{llccccc}
\toprule
Model & Method & Model Inference (s) & Decoding and Sampling (s) & Other (s) & Total (s) & ms/token \\
\midrule
\multirow{5}{*}{Qwen2.5-1.5B}
& Top-$p$ & 233.20 & 4.29 & 0.99 & 238.48 & 18.63 \\
& Min-$p$ & 233.19 & 2.17 & 0.57 & 235.93 & 18.43 \\
& $p$-less & 233.23 & 2.17 & 0.63 & 236.03 & 18.44 \\
& Top-$H$ & 233.56 & 4.68 & 1.00 & 239.25 & 18.69 \\
& \textbf{ME-Decoding} & 242.85 & 5.64 & 1.03 & 249.52 & 19.49 \\
\midrule
\multirow{5}{*}{Qwen3-4B}
& Top-$p$ & 358.20 & 2.89 & 0.50 & 361.58 & 28.25 \\
& Min-$p$ & 356.42 & 1.97 & 0.55 & 358.94 & 28.04 \\
& $p$-less & 356.63 & 1.97 & 0.59 & 359.18 & 28.06 \\
& Top-$H$ & 359.14 & 2.91 & 0.63 & 362.68 & 28.33 \\
& \textbf{ME-Decoding} & 360.29 & 3.01 & 1.15 & 364.45 & 28.47 \\
\bottomrule
\end{tabular}}
\caption{End-to-end GPU timing breakdown.}
\label{tab:gpu_timing_full}
\end{table*}
\FloatBarrier

Table~\ref{tab:gpu_timing_full} shows that ME-Decoding increases total latency by approximately $5.8\%$ on Qwen2.5-1.5B and $1.5\%$ on Qwen3-4B relative to the fastest baseline, with model inference remaining the dominant cost.

\paragraph{Detailed Complexity Analysis.}
\label{app:complexity}

Let $N$ denote the size of the candidate token pool, $d$ denote the token embedding dimension, and $\tau$ denote the number of tokens selected before the early stopping criterion is triggered.
ME-Decoding does not explicitly construct the full $N\times N$ similarity matrix.
The adaptive bandwidth can also be computed without any pairwise construction.
Since we use normalized token embeddings and $\boldsymbol{C}_{ij}=1-\boldsymbol{e}_{i}^\top \boldsymbol{e}_{j}$, we have
\[
\epsilon
=
\frac{1}{2}\sum_{i,j\in\mathcal V}\boldsymbol{p}_{i} \boldsymbol{p}_{j} \boldsymbol{C}_{ij}
=
\frac{1}{2}
\left(
1-
\left\|
\sum_{i\in\mathcal V}\boldsymbol{p}_{i} \boldsymbol{e}_{i}
\right\|_2^2
\right),
\]
where $p$ is normalized over the candidate pool.
Therefore, computing $\epsilon$ only requires a probability-weighted average of token embeddings, with complexity $O(Nd)$.

After obtaining $\epsilon$, kernel entries are dynamically computed according to Eq.~\ref{eq:sim kernel}.
Whenever a new token is selected, the algorithm computes and caches its similarities to all candidate tokens, which costs $O(Nd)$ per selected token and therefore $O(N\tau d)$ in total.
For the greedy selection stage, at the $t$-th step, the selected set has size $t$, and the algorithm evaluates at most $N-t$ remaining candidates.
For each candidate $j$, the dominant operation is computing $\boldsymbol{\alpha}_j=\mathbf R\boldsymbol{\beta}_j$, where $\mathbf R\in\mathbb R^{t\times t}$ and $\boldsymbol{\beta}_j\in\mathbb R^t$, which costs $O(t^2)$.
Hence, the total greedy evaluation cost before stopping is
\[
    \sum_{t=1}^{\tau-1} O\bigl((N-t)t^2\bigr)
    \leq
    \sum_{t=1}^{\tau-1} O(Nt^2)
    =
    O(N\tau^3).
\]
After each selected token, updating $\mathbf R$ and $z$ through the block update costs $O(t^2)$ at step $t$, giving an additional $O(\tau^3)$ cost, which is dominated by $O(N\tau^3)$ since $\tau\leq N$.
Therefore, the overall time complexity of ME-Decoding is
\[
    O(Nd+N\tau d+N\tau^3)
    =
    O\bigl(N\tau(d+\tau^2)\bigr).
\]
The memory overhead is $O(N\tau)$ for cached kernel entries and $O(\tau^2)$ for maintaining $\mathbf R$.
In contrast, explicitly constructing the full kernel matrix would require $O(N^2d)$ time and $O(N^2)$ memory.
Since the early stopping rule usually yields $\tau\ll N$, ME-Decoding avoids the quadratic dependence on the candidate pool size and remains efficient in practice.

\FloatBarrier
\subsection{Hyperparameter Analysis}

We study the sensitivity of ME-Decoding to the hyperparameter $\lambda$, which controls the compactness of the selected token set. A larger $\lambda$ imposes a stronger penalty on the subset size and therefore encourages a tighter candidate set, while a smaller $\lambda$ allows more tokens to be retained. Table~\ref{tab:hyper_rg_details} reports the accuracy of ME-Decoding under different values of $\lambda$ across three models, three temperatures, and two reasoning datasets. The results show that ME-Decoding is relatively stable over a range of $\lambda$ values and consistently achieves strong performance under different settings. Among the tested values, $\lambda=0.9$ obtains the best average accuracy on both GSM8K and GPQA. Therefore, unless otherwise specified, we use $\lambda=0.9$ as the default setting in the main experiments.

\begin{table*}[t!] 
\centering
\footnotesize
\setlength{\tabcolsep}{4.5pt}
\renewcommand{\arraystretch}{1.05}
\begin{tabular}{lccc|ccc|ccc|c}
\toprule
\multirow{2}{*}{\textbf{$\lambda$}}
& \multicolumn{3}{c|}{Qwen3-4B-Inst.}
& \multicolumn{3}{c|}{Phi-4-mini-Inst.}
& \multicolumn{3}{c|}{Mistral-7B-Inst.}
& \multirow{2}{*}{Avg.} \\
\cmidrule(lr){2-4} \cmidrule(lr){5-7} \cmidrule(lr){8-10}
& $T=1.0$ & $T=1.5$ & $T=2.0$
& $T=1.0$ & $T=1.5$ & $T=2.0$
& $T=1.0$ & $T=1.5$ & $T=2.0$
& \\
\midrule
\multicolumn{11}{c}{GSM8K} \\
\midrule
1.0
& 81.20 & 81.35 & 79.76
& 83.78 & 84.38 & 82.03
& 52.84 & 52.77 & 53.45
& 72.39 \\

\rowcolor{gray!18}
0.9
& 80.67 & 79.76 & 80.06
& 84.08 & 84.23 & 82.41
& 55.34 & 53.45 & 53.90
& \textbf{72.66} \\

0.8
& 80.44 & 80.06 & 80.44
& 83.70 & 84.08 & 81.65
& 55.19 & 53.45 & 53.53
& 72.50 \\

0.7
& 79.68 & 80.14 & 79.98
& 84.15 & 83.55 & 82.11
& 53.45 & 54.51 & 53.75
& 72.37 \\
\midrule
\multicolumn{11}{c}{GPQA} \\
\midrule
1.0
& 34.82 & 34.15 & 34.82
& 35.04 & 35.27 & 35.71
& 27.90 & 28.79 & 27.23
& 32.64 \\

\rowcolor{gray!18}
0.9
& 35.49 & 35.49 & 35.71
& 34.38 & 36.16 & 33.93
& 28.35 & 28.79 & 28.35
& \textbf{32.96} \\

0.8
& 36.16 & 36.38 & 35.27
& 35.27 & 33.26 & 34.15
& 27.90 & 25.67 & 27.68
& 32.42 \\

0.7
& 36.38 & 35.04 & 34.60
& 33.71 & 31.92 & 33.93
& 27.90 & 29.24 & 26.34
& 32.12 \\
\bottomrule
\end{tabular}
\caption{Detailed hyperparameter selection results for ME-Decoding. Accuracy is reported in percentage. The highlighted row corresponds to the selected setting, \texttt{$\lambda$=0.9}, which achieves the best average performance on both GSM8K and GPQA.}
\label{tab:hyper_rg_details}
\end{table*}

\subsection{Diversity Analysis}
\label{app:support_diagnostics}

\paragraph{Support-Level Diversity Diagnostics.}

\begin{table*}[t]
\centering
\setlength{\tabcolsep}{3.5pt}
\resizebox{\textwidth}{!}{
\begin{tabular}{llccccccc}
\toprule
Dataset & Method & Acc. (\%) & Mass & Cos. & Avg. $|S|$ & $H_{\mathrm{before}}$ & $H_{\mathrm{after}}$ & $\MES$ \\
\midrule
\multirow{7}{*}{GSM8K}
& \textbf{ME-Decoding} & $\mathbf{63.20\pm0.57}$ & 0.846 & 0.151 & 1.076 & 0.614 & 0.052 & 0.756 \\
& Top-$W$ & $61.64\pm1.19$ & 0.841 & 0.208 & 1.093 & 0.692 & 0.057 & 0.688 \\
& $p$-less & $61.33\pm0.23$ & 0.842 & 0.179 & 1.164 & 0.659 & 0.075 & 0.689 \\
& Top-$H$ & $61.16\pm0.42$ & 0.847 & 0.241 & 2.240 & 0.736 & 0.191 & 0.660 \\
& Approx. size-matched & $61.79\pm1.26$ & 0.850 & 0.210 & 1.067 & 0.615 & 0.044 & 0.707 \\
& Approx. entropy-matched & $61.79\pm0.99$ & 0.851 & 0.208 & 1.070 & 0.614 & 0.045 & 0.708 \\
& Approx. cosine-matched & $62.70\pm0.35$ & 0.840 & 0.198 & 1.014 & 0.598 & 0.010 & 0.715 \\
\midrule
\multirow{7}{*}{GPQA}
& \textbf{ME-Decoding} & $\mathbf{32.66\pm1.05}$ & 0.759 & 0.243 & 1.134 & 0.983 & 0.091 & 0.589 \\
& Top-$W$ & $27.69\pm0.91$ & 0.692 & 0.185 & 1.213 & 1.454 & 0.124 & 0.498 \\
& $p$-less & $28.21\pm1.05$ & 0.770 & 0.192 & 1.265 & 1.002 & 0.144 & 0.553 \\
& Top-$H$ & $27.52\pm0.66$ & 0.778 & 0.184 & 2.877 & 1.277 & 0.431 & 0.475 \\
& Approx. size-matched & $28.47\pm1.17$ & 0.760 & 0.289 & 1.106 & 0.947 & 0.069 & 0.558 \\
& Approx. entropy-matched & $28.47\pm1.17$ & 0.760 & 0.289 & 1.106 & 0.947 & 0.069 & 0.558 \\
& Approx. cosine-matched & $29.25\pm0.91$ & 0.737 & 0.259 & 1.084 & 1.018 & 0.055 & 0.537 \\
\bottomrule
\end{tabular}}
\caption{Complete support diagnostics on Qwen2.5-1.5B at $T=1.0$. Cosine similarity is computed within the selected support at each decoding step and then averaged across steps. The approximately matched Min-$p$ controls are calibrated on independent prompts. Accuracy reports the mean $\pm$ sample standard deviation over three seeds.}
\label{tab:support_matched_controls}
\end{table*}

Table~\ref{tab:support_matched_controls} shows that ME-Decoding achieves the highest accuracy and $\MES$ on both datasets, while retaining a compact support and comparable probability mass.
None of the approximately matched probability-only controls reproduces both results, indicating that the gain is not explained solely by stronger truncation or one support statistic.


\paragraph{Detailed Diversity Results.}
Table~\ref{tab:diversity_temperature} reports the detailed numerical results corresponding to the diversity analysis in the main text. 
For each temperature, we report Acc, Distinct-1, Distinct-2, $1-\mathrm{Self\text{-}BLEU}$, and the aggregated Diversity Avg. score.
The diversity score is computed by min-max normalizing the three diversity metrics within each temperature and then taking their average.
Although Top-$p$ often obtains the highest diversity score, especially at high temperature, its accuracy drops substantially.
In contrast, ME-Decoding maintains the strongest accuracy across temperatures, showing that it provides a better accuracy--diversity trade-off for reasoning-oriented generation.

\begin{table*}[t]
\centering
\setlength{\tabcolsep}{4.5pt}
\renewcommand{\arraystretch}{1.05}
\caption{
Accuracy and diversity comparison under different temperatures.
Diversity Avg. is computed by first min-max normalizing Distinct-1, Distinct-2, and $1-\mathrm{Self\text{-}BLEU}$ within each temperature, and then averaging the three normalized scores.
}
\label{tab:diversity_temperature}
\resizebox{\textwidth}{!}{
\begin{tabular}{c l c c c c c}
\toprule
$T$
& Method
& Acc $\uparrow$ 
& Distinct-1 $\uparrow$
& Distinct-2 $\uparrow$
& $1-\mathrm{Self\text{-}BLEU}$ $\uparrow$ 
& Diversity Avg. $\uparrow$ \\
\midrule
\multirow{8}{*}{1.0}
& Top-$p$                         & 0.8484 & 0.0050 & 0.0695 & 0.2570 & 0.9842 \\
& Min-$p$                         & 0.8596 & 0.0050 & 0.0705 & 0.2595 & 1.0000 \\
& $p$-less                        & 0.8906 & 0.0046 & 0.0445 & 0.0607 & 0.1012 \\
& Top-$H$                         & 0.8832 & 0.0045 & 0.0418 & 0.0614 & 0.0043 \\
& Top-$W$                         & 0.8996 & 0.0049 & 0.0480 & 0.0686 & 0.3550 \\
\rowcolor{gray!18} 
& ME-Decoding ($\lambda=0.9$)     & 0.9030 & 0.0049 & 0.0472 & 0.0588 & 0.3294 \\
\rowcolor{gray!18} 
& ME-Decoding ($\lambda=0.8$)     & 0.9008 & 0.0049 & 0.0504 & 0.0813 & 0.4039 \\
\rowcolor{gray!18} 
& ME-Decoding ($\lambda=0.7$)     & 0.9024 & 0.0049 & 0.0528 & 0.1014 & 0.4652 \\
\midrule
\multirow{8}{*}{1.5}
& Top-$p$                         & 0.6794 & 0.0224 & 0.1503 & 0.4136 & 1.0000 \\
& Min-$p$                         & 0.7802 & 0.0054 & 0.0849 & 0.3421 & 0.3884 \\
& $p$-less                        & 0.8682 & 0.0046 & 0.0546 & 0.1464 & 0.0783 \\
& Top-$H$                         & 0.8224 & 0.0047 & 0.0585 & 0.1974 & 0.1434 \\
& Top-$W$                         & 0.8982 & 0.0049 & 0.0542 & 0.1148 & 0.0516 \\
\rowcolor{gray!18} 
& ME-Decoding ($\lambda=0.9$)     & 0.9026 & 0.0052 & 0.0526 & 0.0735 & 0.0112 \\
\rowcolor{gray!18} 
& ME-Decoding ($\lambda=0.8$)     & 0.9058 & 0.0051 & 0.0564 & 0.1046 & 0.0528 \\
\rowcolor{gray!18} 
& ME-Decoding ($\lambda=0.7$)     & 0.8996 & 0.0053 & 0.0592 & 0.1268 & 0.0879 \\
\midrule
\multirow{8}{*}{2.0}
& Top-$p$                         & 0.1886 & 0.2494 & 0.7750 & 0.9054 & 1.0000 \\
& Min-$p$                         & 0.6600 & 0.0068 & 0.1122 & 0.4487 & 0.1833 \\
& $p$-less                        & 0.8128 & 0.0049 & 0.0659 & 0.2454 & 0.0786 \\
& Top-$H$                         & 0.7394 & 0.0053 & 0.0784 & 0.3209 & 0.1149 \\
& Top-$W$                         & 0.8906 & 0.0047 & 0.0543 & 0.1397 & 0.0310 \\
\rowcolor{gray!18} 
& ME-Decoding ($\lambda=0.9$)     & 0.9068 & 0.0050 & 0.0499 & 0.0668 & 0.0004 \\
\rowcolor{gray!18} 
& ME-Decoding ($\lambda=0.8$)     & 0.9036 & 0.0051 & 0.0553 & 0.0959 & 0.0146 \\
\rowcolor{gray!18} 
& ME-Decoding ($\lambda=0.7$)     & 0.9034 & 0.0052 & 0.0573 & 0.1187 & 0.0247 \\
\bottomrule
\end{tabular}
}
\end{table*}

\subsection{Additional Results for Open-Ended Generation}
\label{app:instruction_chat_details}
\label{app:judge_robustness}

In the main text, we report the averaged instruction-following and chat performance over three instruction-tuned models.
Here, we provide the corresponding per-model visualizations and detailed numerical results.
Figure~\ref{fig:instruction_chat_detailed} shows the performance curves on MT-Bench and AlpacaEval across temperatures for each model.
Figure~\ref{fig:exp002_rank_heatmap} further summarizes the relative ranking of each decoding method under every model-temperature setting, averaged over MT-Bench and AlpacaEval.
Lower ranks indicate better performance.
The heatmap shows that ME-Decoding obtains competitive ranks in multiple settings and achieves the best overall average rank, suggesting that its improvement is stable across models and temperatures.
Tables~\ref{tab:exp002_mt_mean_temperature} and~\ref{tab:exp002_alpacapct_temperature} report the exact MT-Bench judge scores and AlpacaEval candidate win-rates, respectively.
ME-Decoding achieves the highest overall averages, with $7.17$ on MT-Bench and $14.80\%$ on AlpacaEval, although individual baselines lead in several settings.
All results are evaluated using the same judge-based evaluation protocol as in the main text and aggregated over 3 runs.

\paragraph{Evaluation Robustness.}
AlpacaEval-style evaluation compares each generated response with a fixed reference for the same prompt, while MT-Bench-style evaluation independently assigns a score from 1 to 10.
All methods use matched samples at $T\in\{1.0,1.5,2.0\}$ over three generation runs.
We evaluate the same outputs with DeepSeek-V4-Pro and GLM-5.2 \citep{zeng2026glm}.
Table~\ref{tab:two_judges} shows that ME-Decoding achieves the highest result under both judges and evaluation protocols.
Table~\ref{tab:paired_bootstrap} further reports positive paired $95\%$ confidence intervals against Top-$W$, $p$-less, and Top-$H$, supporting the consistency of the gains beyond aggregate means.
Additional length, style, and refusal diagnostics indicate that the gains are not explained by verbosity or conservative refusal behavior.

\begin{table*}[t]
\centering
\footnotesize
\setlength{\tabcolsep}{7pt}
\begin{tabular}{llcc}
\toprule
Judge & Method & AlpacaEval (\%) & MT-Bench \\
\midrule
\multirow{6}{*}{DeepSeek-V4-Pro}
& Top-$p$ & $13.04\pm2.07$ & $6.34\pm0.89$ \\
& Min-$p$ & $14.14\pm0.60$ & $6.86\pm0.41$ \\
& Top-$H$ & $14.36\pm0.58$ & $6.88\pm0.33$ \\
& $p$-less & $14.33\pm0.07$ & $6.94\pm0.11$ \\
& Top-$W$ & $14.14\pm0.04$ & $7.08\pm0.02$ \\
& \textbf{ME-Decoding} & $\mathbf{14.80\pm0.30}$ & $\mathbf{7.17\pm0.03}$ \\
\midrule
\multirow{4}{*}{GLM-5.2}
& Top-$H$ & $9.48\pm0.78$ & $5.42\pm0.23$ \\
& $p$-less & $10.18\pm0.24$ & $5.54\pm0.05$ \\
& Top-$W$ & $10.11\pm0.20$ & $5.58\pm0.03$ \\
& \textbf{ME-Decoding} & $\mathbf{10.63\pm0.16}$ & $\mathbf{5.64\pm0.01}$ \\
\bottomrule
\end{tabular}
\caption{Open-ended evaluation under two automatic judges. Results average temperature-level aggregates over three generation runs.}
\label{tab:two_judges}
\end{table*}

\begin{table*}[t]
\centering
\footnotesize
\setlength{\tabcolsep}{6pt}
\begin{tabular}{llcccc}
\toprule
Judge & Baseline & $\Delta$ AlpacaEval (pp) & 95\% CI & $\Delta$ MT-Bench & 95\% CI \\
\midrule
\multirow{3}{*}{DeepSeek-V4-Pro}
& Top-$W$ & +0.656 & $[0.320,0.980]$ & +0.092 & $[0.018,0.169]$ \\
& $p$-less & +0.469 & $[0.115,0.810]$ & +0.234 & $[0.158,0.305]$ \\
& Top-$H$ & +0.440 & $[0.076,0.799]$ & +0.297 & $[0.221,0.373]$ \\
\midrule
\multirow{3}{*}{GLM-5.2}
& Top-$W$ & +0.529 & $[0.239,0.819]$ & +0.069 & $[0.018,0.120]$ \\
& $p$-less & +0.456 & $[0.138,0.782]$ & +0.107 & $[0.052,0.163]$ \\
& Top-$H$ & +1.157 & $[0.824,1.481]$ & +0.221 & $[0.162,0.281]$ \\
\bottomrule
\end{tabular}
\caption{Paired-bootstrap gains of ME-Decoding. Resampling preserves prompt, model, temperature, and generation-run matching.}
\label{tab:paired_bootstrap}
\end{table*}

\begin{figure*}[t]
    \centering
    \includegraphics[width=\textwidth]{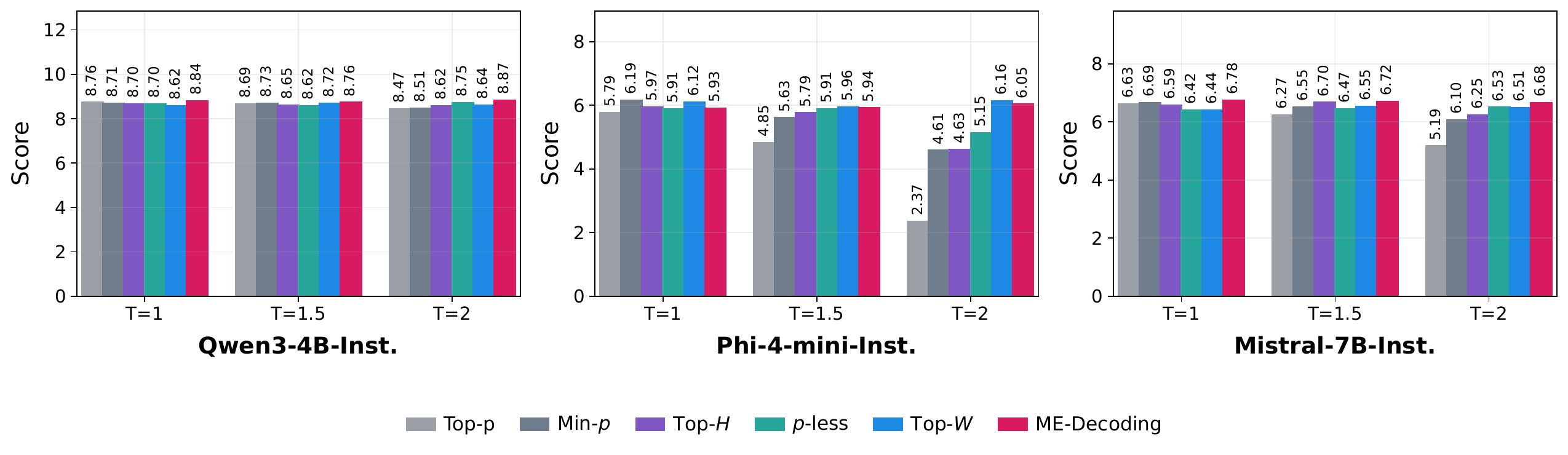}

    \includegraphics[width=\textwidth]{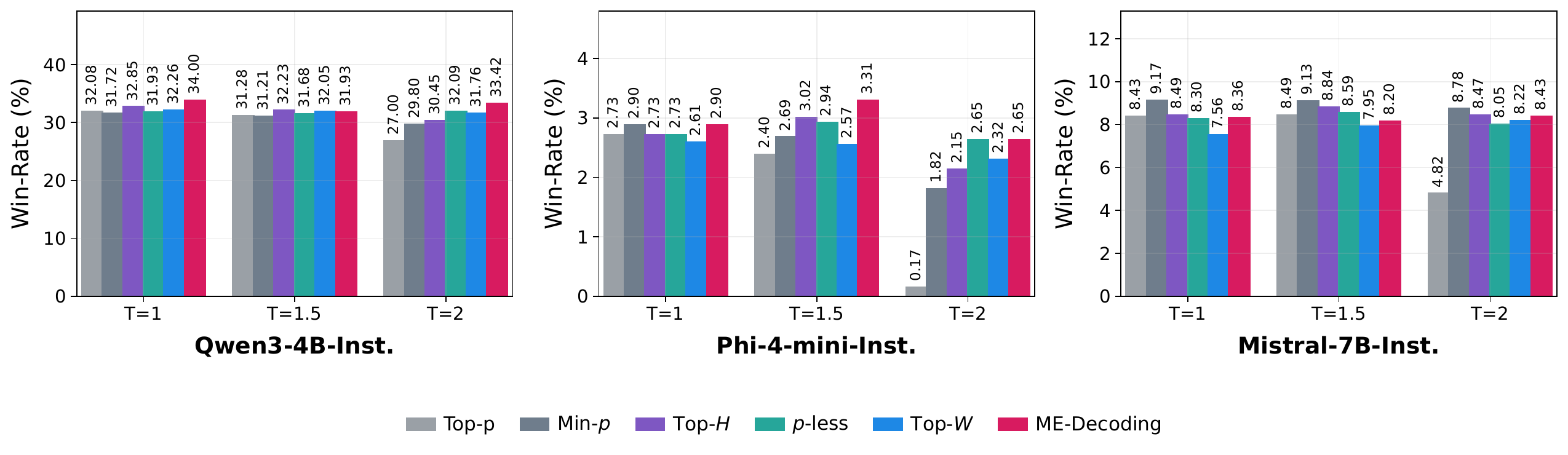}

    \caption{Detailed instruction-following and chat performance across temperatures. Top: MT-Bench judge scores. Bottom: AlpacaEval candidate win-rate. Results are reported for three instruction-tuned models and aggregated over 3 runs.}
\label{fig:instruction_chat_detailed}
\end{figure*}

\begin{table*}[t]
\centering
\footnotesize
\setlength{\tabcolsep}{5pt}
\begin{tabular}{lccc|ccc|ccc|c}
\toprule
& \multicolumn{3}{c|}{Qwen3-4B-Inst.}
& \multicolumn{3}{c|}{Phi-4-mini-Inst.}
& \multicolumn{3}{c|}{Mistral-7B-Inst.}
& \\
\cmidrule(lr){2-4} \cmidrule(lr){5-7} \cmidrule(lr){8-10}
Method
& $T=1.0$ & $T=1.5$ & $T=2.0$
& $T=1.0$ & $T=1.5$ & $T=2.0$
& $T=1.0$ & $T=1.5$ & $T=2.0$
& Avg. \\
\midrule
Min-$p$
& 8.71 & \underline{8.73} & 8.51
& \textbf{6.19} & 5.63 & 4.61
& \underline{6.69} & 6.55 & 6.10
& 6.86 \\

Top-$p$
& \underline{8.76} & 8.69 & 8.47
& 5.79 & 4.85 & 2.37
& 6.63 & 6.27 & 5.19
& 6.34 \\

$p$-less
& 8.70 & 8.62 & \underline{8.75}
& 5.91 & 5.91 & 5.15
& 6.42 & 6.47 & \underline{6.53}
& 6.94 \\

Top-$H$
& 8.70 & 8.65 & 8.62
& 5.97 & 5.79 & 4.63
& 6.59 & \underline{6.70} & 6.25
& 6.88 \\

Top-$W$
& 8.62 & 8.72 & 8.64
& \underline{6.12} & \textbf{5.96} & \textbf{6.16}
& 6.44 & 6.55 & 6.51
& \underline{7.08} \\

\rowcolor{gray!18} \textbf{Ours}
& \textbf{8.84} & \textbf{8.76} & \textbf{8.87}
& 5.93 & \underline{5.94} & \underline{6.05}
& \textbf{6.78} & \textbf{6.72} & \textbf{6.68}
& \textbf{7.17} \\

\bottomrule
\end{tabular}
\caption{Detailed MT-Bench judge scores. Results are reported across three instruction-tuned models and three temperatures, aggregated over 3 runs. The Avg. column gives the unweighted average over all model-temperature combinations. Best results are shown in \textbf{bold}, and second-best results are \underline{underlined}.}
\label{tab:exp002_mt_mean_temperature}
\end{table*}

\begin{table*}[t]
\centering
\footnotesize
\setlength{\tabcolsep}{5pt}
\begin{tabular}{lccc|ccc|ccc|c}
\toprule
& \multicolumn{3}{c|}{Qwen3-4B-Inst.}
& \multicolumn{3}{c|}{Phi-4-mini-Inst.}
& \multicolumn{3}{c|}{Mistral-7B-Inst.}
& \\
\cmidrule(lr){2-4} \cmidrule(lr){5-7} \cmidrule(lr){8-10}
Method
& $T=1.0$ & $T=1.5$ & $T=2.0$
& $T=1.0$ & $T=1.5$ & $T=2.0$
& $T=1.0$ & $T=1.5$ & $T=2.0$
& Avg. \\
\midrule
Min-$p$
& 31.72 & 31.21 & 29.80
& \textbf{2.90} & 2.69 & 1.82
& \textbf{9.17} & \textbf{9.13} & \textbf{8.78}
& 14.14 \\

Top-$p$
& 32.08 & 31.28 & 27.00
& 2.73 & 2.40 & 0.17
& 8.43 & 8.49 & 4.82
& 13.04 \\

$p$-less
& 31.93 & 31.68 & \underline{32.09}
& \underline{2.73} & 2.94 & \textbf{2.65}
& 8.30 & 8.59 & 8.05
& 14.33 \\

Top-$H$
& \underline{32.85} & \textbf{32.23} & 30.45
& \underline{2.73} & \underline{3.02} & 2.15
& \underline{8.49} & \underline{8.84} & \underline{8.47}
& \underline{14.36} \\

Top-$W$
& 32.26 & \underline{32.05} & 31.76
& 2.61 & 2.57 & \underline{2.32}
& 7.56 & 7.95 & 8.22
& 14.14 \\

\rowcolor{gray!18} \textbf{Ours}
& \textbf{34.00} & 31.93 & \textbf{33.42}
& \textbf{2.90} & \textbf{3.31} & \textbf{2.65}
& 8.36 & 8.20 & 8.43
& \textbf{14.80} \\

\bottomrule
\end{tabular}
\caption{Detailed AlpacaEval candidate win-rate (\%). Results are reported across three instruction-tuned models and three temperatures, aggregated over 3 runs. The Avg. column gives the unweighted average over all model-temperature combinations. Best results are shown in \textbf{bold}, and second-best results are \underline{underlined}.}
\label{tab:exp002_alpacapct_temperature}
\end{table*}

\begin{figure*}[htbp]
    \centering
    \includegraphics[width=\textwidth]{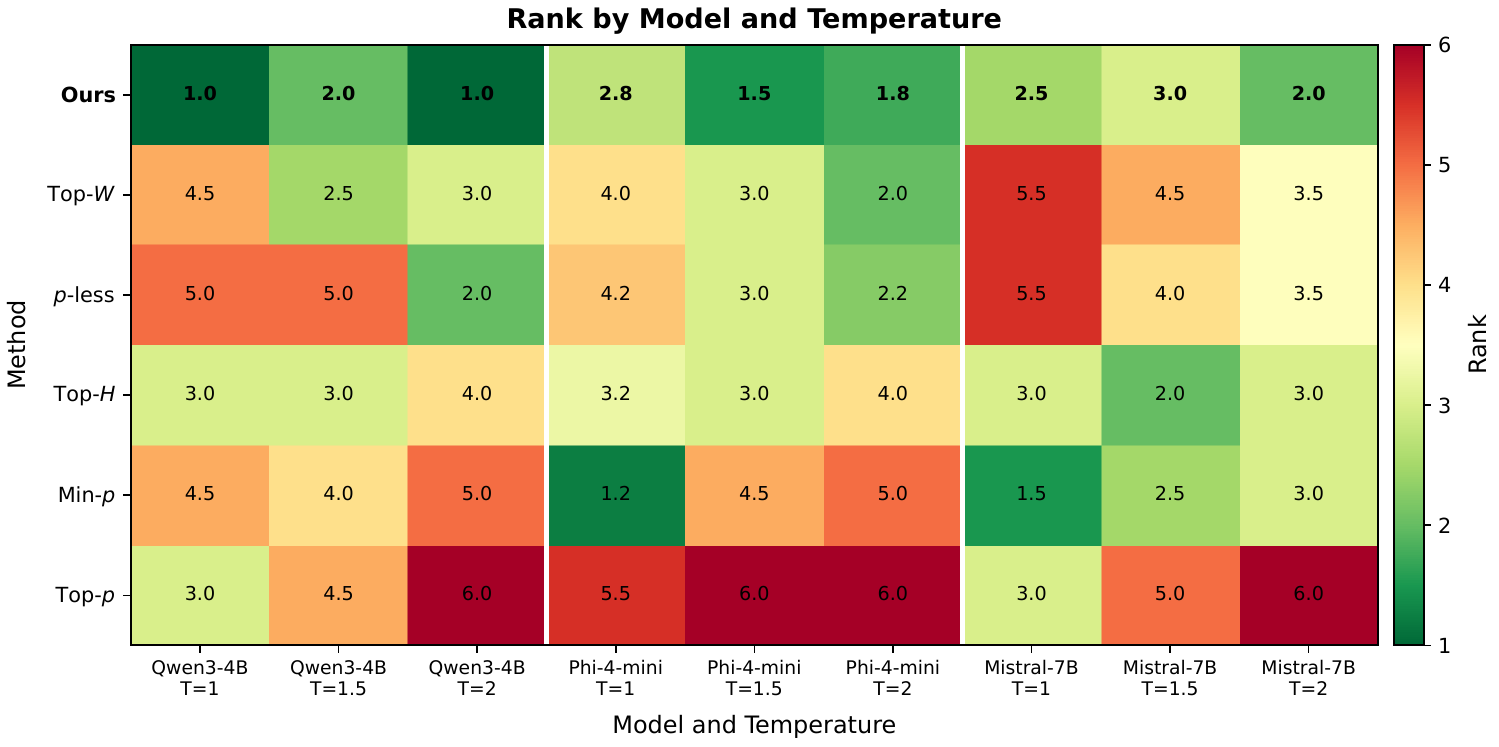}
    \caption{
    Detailed average-rank comparison on open-ended generation benchmarks.
    The heatmap reports the rank of each decoding method under each model-temperature setting, averaged over AlpacaEval and MT-Bench.
    }
    \label{fig:exp002_rank_heatmap}
\end{figure*}